\documentclass{article}

\usepackage{microtype}
\usepackage{graphicx}
\usepackage{booktabs} %

\usepackage{nicefrac}
\usepackage{lipsum}
\usepackage{ulem}
\usepackage[labelfont=bf]{caption}
\usepackage{subcaption}
\usepackage{wrapfig}
\usepackage{xcolor}
\usepackage[export]{adjustbox}
\usepackage{comment}

\usepackage{amsmath,amsfonts,bm}

\def\eqref#1{equation~\ref{#1}}

\def\floor#1{\lfloor #1 \rfloor}
\def\1{\bm{1}}

\DeclareMathAlphabet{\mathsfit}{\encodingdefault}{\sfdefault}{m}{sl}
\SetMathAlphabet{\mathsfit}{bold}{\encodingdefault}{\sfdefault}{bx}{n}

\newcommand{\E}{\mathbb{E}}

\newcommand{\R}{\mathbb{R}}

\newcommand{\hgamma}{\mathsf{h}_{\gamma}}
\newcommand{\hgam}[1]{\mathsf{h}_{#1}}
\newcommand{\thetap}{\theta_{\mathrm{+}}}

\newcommand{\scal}[2]{\left\langle{#1},{#2}\right\rangle}

\newcommand{\LL}{\mathcal{L}}
\renewcommand{\eqref}[1]{\textrm{Eq}.(\ref{#1})}

\newcommand{\ggamma}{\mathsf{g}_{\gamma}}
\newcommand{\ggam}[1]{\mathsf{g}_{#1}}

\newcommand{\Ex}[1]{{\mathbb E}\left[#1\right]}

\newcommand{\invl}{\mathrm{I}}
\newcommand{\itvl}[1]{\mathrm{I}_{#1}}

\def\R{\mathbb{R}}
\def\N{\mathbb{N}}

\newcommand{\updicml}[1]{{#1}}

\newcommand{\rmv}[1]{}

\newcommand{\upd}[1]{{#1}}

\usepackage{tcolorbox}

\newcommand{\eqals}[1]{\begin{align}#1\end{align}}

\graphicspath{ {images/} }

\usepackage[space]{grffile}
\usepackage{booktabs} %
\usepackage{stmaryrd}
\usepackage[utf8]{inputenc}
\usepackage{enumitem}
\setitemize{topsep=2pt,parsep=0pt,partopsep=0pt,leftmargin=25pt}

\newcommand{\myparagraph}{\textbf}

\usepackage{hyperref}

\usepackage[accepted]{icml2023}

\usepackage{amsmath}
\usepackage{amssymb}
\usepackage{mathtools}
\usepackage{amsthm}

\usepackage[capitalize,noabbrev]{cleveref}

\theoremstyle{plain}
\newtheorem{theorem}{Theorem}[section]
\newtheorem{proposition}[theorem]{Proposition}
\newtheorem{lemma}[theorem]{Lemma}

\theoremstyle{definition}

\theoremstyle{remark}

\newtheorem{property}[theorem]{Property}

\newtheorem{innercustomproposition}{Proposition}
\newenvironment{customproposition}[1]{\renewcommand\theinnercustomproposition{#1}\innercustomproposition}{\endinnercustomproposition}

\usepackage[textsize=tiny]{todonotes}

\begin{document}

\twocolumn[

\icmltitle{SGD with Large Step Sizes Learns Sparse Features}

\icmlsetsymbol{equal}{*}

\begin{icmlauthorlist}
\icmlauthor{Maksym Andriushchenko}{epfl}
\icmlauthor{Aditya Varre}{epfl}
\icmlauthor{Loucas Pillaud-Vivien}{epfl}
\icmlauthor{Nicolas Flammarion}{epfl}
\end{icmlauthorlist}

\icmlaffiliation{epfl}{EPFL}

\icmlcorrespondingauthor{Maksym Andriushchenko}{maksym.andriushchenko@epfl.ch}

\icmlkeywords{Machine Learning, ICML}

\vskip 0.3in
]

\printAffiliationsAndNotice{}  %

\begin{abstract}
We showcase important features of the dynamics of the Stochastic Gradient Descent (SGD) in the training of neural networks. We present empirical observations that commonly used large step sizes (i) may lead the iterates to jump from one side of a valley to the other causing \textit{loss stabilization}, and (ii) this stabilization induces a hidden stochastic dynamics that \textit{biases it implicitly} toward \textit{sparse} predictors. Furthermore, we show empirically that the longer large step sizes keep SGD high in the loss landscape valleys, the better the implicit regularization can operate and find sparse representations. Notably, no explicit regularization is used: the regularization effect comes solely from the SGD dynamics influenced by the large step sizes schedule. Therefore, these observations unveil how, through the step size schedules, both gradient and noise drive together the SGD dynamics through the loss landscape of neural networks. We justify these findings theoretically through the study of simple neural network models as well as qualitative arguments inspired from stochastic processes. This analysis allows us to shed new light on some common practices and observed phenomena when training deep networks. The code of our experiments is available at \url{https://github.com/tml-epfl/sgd-sparse-features}.
\end{abstract}
\vspace{-5mm}
\section{Introduction}  %

Deep neural networks have accomplished remarkable achievements on a wide variety of tasks.
Yet, the understanding of their remarkable effectiveness remains incomplete. 
From an optimization perspective, stochastic training procedures challenge many insights drawn from convex models. 
{Notably,} large step size schedules used in practice %
lead to unexpected patterns of stabilizations and sudden drops in the training loss \citep{he2016deep}.
From a generalization perspective, overparametrized deep nets generalize well while fitting perfectly the data and without any explicit regularizers~\citep{zhang2016understanding}. 
This suggests that optimization and generalization are tightly intertwined: neural networks find solutions that generalize well \textit{thanks} to the optimization procedure used to train them.
This property, known as \textit{implicit bias} or \textit{algorithmic regularization}, has been studied  both for regression~\citep{li2018algorithmic,woodworth2020kernel} and classification~\citep{soudry2018implicit,Lyu2020Gradient,chizat2020implicit}. 
However, for these theoretical results, \upd{it is also shown that} typical timescales needed to enter the beneficial feature learning regimes are prohibitively long \upd{\citep{woodworth2020kernel, moroshko2020implicit}.}

In this paper, \upd{we aim at staying closer to the experimental practice and consider the SGD schedules from the ResNet paper~\citep{he2016deep}} where the \textit{large step size} is first kept constant and then decayed, potentially multiple times. 
We illustrate this behavior in Fig.~\ref{fig:fig1_deep_nets_with_wd} where we reproduce a minimal setting without data augmentation or momentum, and with only one step size decrease. 
We draw attention to two key observations regarding the large step size phase: (a) quickly after the start of training, the loss remains approximately constant \upd{on average} and (b) despite no progress on the training loss, running this phase for longer leads to better generalization.
We refer to such large step size phase as \textit{loss stabilization}. \upd{The better generalization hints at some \textit{hidden dynamics} in the parameter space not captured by the loss curves in Fig.~\ref{fig:fig1_deep_nets_with_wd}.}
Our main contribution is to unveil the hidden dynamics behind this phase: loss stabilization helps to amplify the noise of SGD that drives the network towards a solution with \textit{sparser features}, meaning that for a feature vector $\psi(x)$, only a few unique features are active for a given input $x$.

\begin{figure*}[t]
    \centering
    \includegraphics[width=0.4\textwidth]{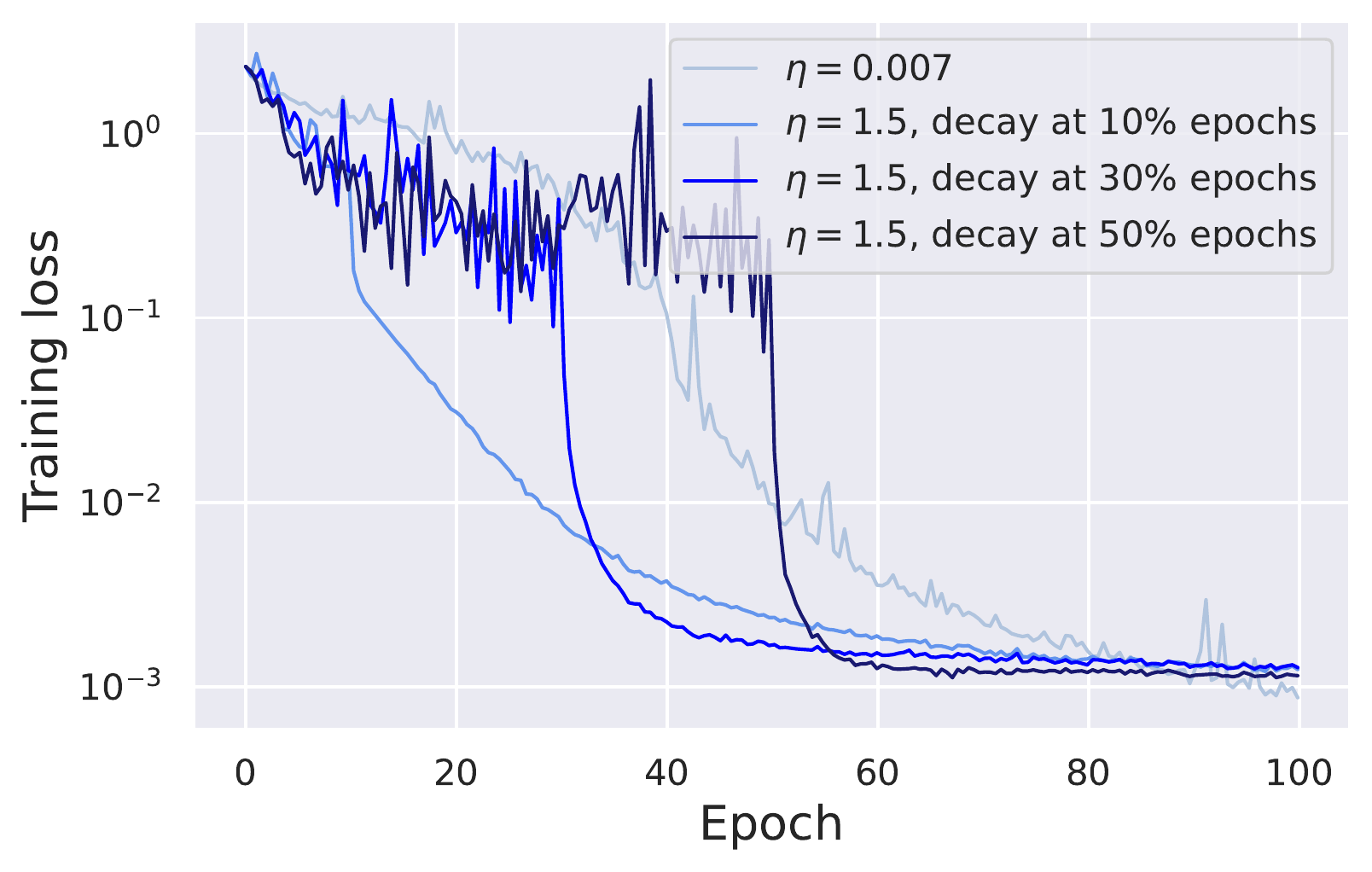}
    \quad \quad
    \includegraphics[width=0.4\textwidth]{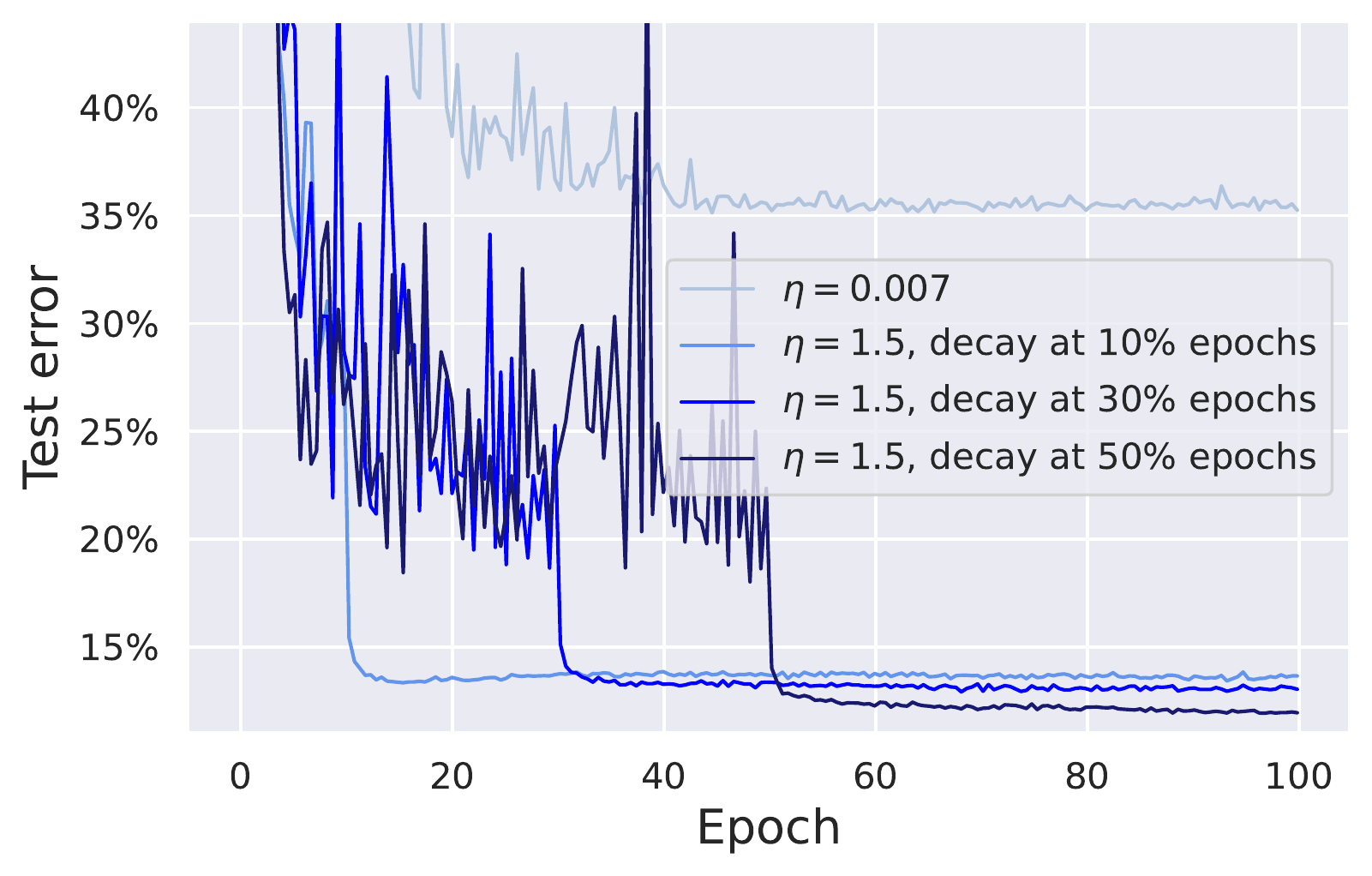}
    \vspace{-4mm}
    \caption{\textbf{A typical training dynamics for a ResNet-18 trained on CIFAR-10.} We use weight decay but no momentum or data augmentation for this experiment. We see a substantial difference in generalization (as large as 12\% vs. 35\% test error) depending on the step size $\eta$ and its schedule. When the training loss stabilizes, there is a hidden progress occurring which we aim to characterize.} %
    \vspace{-3mm}
    \label{fig:fig1_deep_nets_with_wd}
\end{figure*}

\subsection{Our Contributions}

\myparagraph{The effective dynamics behind loss stabilization.} We characterize two main components of the SGD dynamics with large step sizes: (i) a fast movement determined by the bouncing directions causing loss stabilization, (ii) a slow dynamics driven by the combination of the gradient and 
 the multiplicative noise---which is non-vanishing due to the loss stabilization.

\upd{\myparagraph{SDE model and sparse feature learning.} We model the \textit{effective} slow dynamics during loss stabilization by a stochastic differential equation (SDE) whose multiplicative noise is related to the neural tangent kernel features, and validate this modeling experimentally. Building on the existing theory on diagonal linear networks, which shows that this noise structure leads to sparse predictors, we conjecture a similar ``sparsifying'' effect on the features of more complex architectures. We experimentally confirm this on neural networks of increasing complexity.} 
 
\upd{\myparagraph{Insights from our understanding.}  We draw a clear general picture: the hidden optimization dynamics induced by large step sizes and loss stabilization enable the transition to a sparse feature learning regime. We argue that after a short initial phase of training, SGD \textit{first} identifies sparse features of the training data and eventually fits the data when the step size is decreased. Finally, we discuss informally how many deep learning regularization methods (weight decay, BatchNorm, SAM) may also fit into the same picture.}

\subsection{Related Work}

\citet{he2016deep} popularized the piece-wise constant step size schedule which often exhibits a clear loss stabilization pattern which was later characterized theoretically in \citet{li2020reconciling} from the optimization point of view. %
However, the regularization effect of this phase induced by the underlying \textit{hidden stochastic dynamics} is still unclear. 
\citet{li2019towards} analyzed the role of loss stabilization for a synthetic distribution containing different patterns, but it is not clear how this analysis can be extended to general problems. 
\updicml{\citet{jastrzebski21a} suggest that large step sizes prevent the increase of local curvature during the early phase of training. However, they do not provide an explanation for this phenomenon.} 

The importance of large step sizes for generalization has been investigated with diverse motivations. Many works conjectured that large step sizes induce minimization of some complexity measures related to the flatness of minima \citep{keskar2016large, smith2017bayesian, smith2021origin, yang2022stochastic}. Notably, \citet{xing2018walk} point out that SGD moves through the loss landscape bouncing between the walls of a valley where the role of large step sizes is to guide the SGD iterates towards a flatter minimum. 
\updicml{ However, the correct flatness measure is often disputed~\citep{dinh2017sharp} and its role in understanding generalization is questionable since full-batch GD with large step sizes (unlike SGD) can lead to flat solutions which don't generalize well \citep{kaur2022maximum}}

The attempts to explain the effect of large step size on strongly convex models \citep{nakkiran2020learning,wu2021direction,beugnot2022benefits} are inherently incomplete since it is a phenomenon related to the existence of many zero solutions with very different generalization properties. 
Works based on stability analysis characterize the properties of the minimum that SGD or GD can potentially converge depending on the step size \citep{wu2018sgd,mulayoff2021implicit,ma2021sobolev,nacson2022implicit}. However, these approaches \textit{do not capture the entire training dynamics} \upd{such as the large step size phase that we consider where SGD converges only after the step size is decayed.}

\updicml{To grasp the generalization of SGD, research has focused on SGD augmented with label noise due to its  beneficial regularization properties and resemblance to the standard noise of SGD.}
\rmv{To understand the generalization of SGD, SGD \updicml{doped} with label noise has been studied because of its beneficial regularization effect and its resemblance to SGD's standard noise.}Its implicit bias has been first characterized by \citet{blanc2020implicit} and extended by \citet{li2021happens}. However, their analysis only holds in the final phase of the training, close to a zero-loss manifold. Our work instead is closer in spirit to \citet{pillaud2022label} where the label noise dynamics is analyzed in the \textit{central} phase of the training, \upd{i.e., when the loss is still substantially above zero.} 

\updicml{The dynamics of GD with large step sizes have received a lot of attention in recent times, particularly the edge-of-stability phenomenon \cite{cohen2021gradient} and the catapult mechanism \citep{lewkowycz2020large, wang2022large}. However, the lack of stochastic noise in their analysis renders them incapable of capturing stochastic training.} %
Note that it is possible to bridge the gap between GD and SGD by using explicit regularization as in \citet{geiping2021stochastic}. We instead focus on the \textit{implicit} regularization of SGD which remains the most practical approach for training deep nets.

Finally, sparse features and low-rank structures in deep networks have been commonly used for model compression, knowledge distillation, and lottery ticket hypothesis \citep{denton2014exploiting, hinton2015distilling, frankle2018lottery}. A common theme of all these works is the presence of \textit{hidden structure} in the networks learned by SGD which allows one to come up with a much smaller network that approximates well the original one. In particular, \citet{hoefler2021sparsity} note that ReLU activations in deep networks trained with SGD are typically much sparser than 50\%. Our findings suggest that the step size schedule can be the key component behind emergence of such sparsity. %

\section{The Effective Dynamics of SGD with Large Step Size: Sparse Feature Learning}  %
In this section, we show that large step sizes may lead the loss to stabilize by making SGD bounce above a valley. We then unveil the effective dynamics induced by this loss stabilization. To clarify our exposition we showcase our results for the mean square error but other losses like the cross-entropy carry \upd{the same key properties in terms of the noise covariance \citep[Lemma 2.14]{wojtowytsch2021stochastic_discrete}}. We consider a generic parameterized family of prediction functions $\mathcal{H}:=\{x \to h_\theta (x),\  \theta \in \R^p\}$, a setting which encompasses neural networks. %
In this case, the training loss on input/output samples $(x_i, y_i)_{1 \leq i \leq n} \in \R^d \times \R$ is equal to
\begin{align}
\label{eq:general_loss}
    \LL(\theta):= \frac{1}{2n} \sum_{i = 1}^n \left(h_\theta(x_i) - y_i\right)^2.
\end{align}
We consider the overparameterized setting, i.e. $p \gg n$, hence, there shall exists many parameters $\theta^*$ that lead to zero loss, i.e., perfectly interpolate the dataset. Therefore, the question of which interpolator the algorithm converges to is of paramount importance in terms of generalization.
We focus on the SGD recursion with step size $\eta > 0$, initialized at $\theta_0 \in \R^p$: for all $t \in \N$,
\begin{align}
\label{eq:general_SGD}
    \theta_{t+1} = \theta_{t} - \eta (h_{\theta_t}(x_{i_t}) - y_{i_t}) \nabla_\theta h_{\theta_t}(x_{i_t}),
\end{align}
where $i_t \sim \mathcal{U}\left(\llbracket 1,n \rrbracket\right)$ is the uniform distribution over the sample indices. In the following, note that SGD with mini batches of size \upd{$B > 1$} would lead to similar analysis \upd{but with $\eta/B$ instead of $\eta$}.

\subsection{\upd{Background:} SGD is GD with Specific Label Noise}
\label{subsec:effective_dynamics}

To emphasize the combined roles of gradient and noise, we highlight the connection between the SGD dynamics and that of full-batch GD plus a specific label noise. \upd{Such manner of reformulating the dynamics has already been used in previous works attempting to understand the specificity of the SGD noise~\citep{haochen2020shape, ziyin2021strength}. We formalize it in the following proposition.}
\begin{proposition}
\label{prop:SGD_GD_LN}
Let $(\theta_t)_{t \geq 0}$ follow the SGD dynamics \eqref{eq:general_SGD} with the random sampling function $(i_t)_{t \geq 0}$. For $t \geq 0$, define the random vector $\xi_t \in \R^n$ such that
\begin{align}
\label{eq:xi}
[\xi_{t}]_i:= ( h_{\theta_t} (x_i) - y_i) (1 - n \mathbf{1}_{i = i_t}),
\end{align}
for $i \in \llbracket 1,n \rrbracket$ and where $\mathbf{1}_{A}$ is the indicator of the event $A$.
\upd{Then $(\theta_t)_{t \geq 0}$ follows the full-batch gradient dynamics on $\LL$ with label noise $(\xi_{t})_{t \geq 0}$, that is
\begin{align}
\label{eq:explicit_label_noise}
\theta_{t+1} = \theta_t -\frac{\eta}{n} \sum_{i=1}^n (h_{\theta_t}(x_i) - y^t_{i}) \nabla_\theta h_{\theta_t}(x_i),
\end{align}
where we define the random labels $y^t := y + \xi_{t}$. Furthermore, $\xi_t$ is a mean zero random vector with variance such that $\frac{1}{n(n-1)}\E \left\| \xi_t \right\|^2= 2 \LL(\theta_t)$.}
\end{proposition}
\upd{This reformulation shows two crucial aspects of the SGD noise: (i) the noisy part at state $\theta$ always belongs to the linear space spanned by $\{\nabla_\theta h_{\theta}(x_1), \hdots, \nabla_\theta h_{\theta}(x_n)\}$, and (ii) it scales as the training loss. 
Going further on (ii), we highlight in the following section that the loss can stabilize because of large step sizes: this may lead to a \textit{constant} effective scale of label noise. These two features are of paramount importance when modelling the effective dynamics that take place during loss stabilization.}

\subsection{The Effective Dynamics Behind Loss Stabilization} 

\myparagraph{On loss stabilization.} For generic quadratic costs, e.g., $F(\beta) := \| X\beta - y \|^2$, %
gradient descent with step size $\eta$ is convergent for $\eta < 2/\lambda_{\max}$, divergent for $\eta > 2/\lambda_{\max}$ and converges to a bouncing $2$-periodic dynamics for $\eta = 2/\lambda_{\max}$, where $\lambda_{\max}$ is the largest eigenvalue of the Hessian.
However, the practitioner is not likely to hit perfectly this unstable step size and, almost surely, the dynamics shall either converge or diverge.
Yet, non-quadratic costs bring to this picture a particular complexity: it has been shown that, even for non-convex toy models, there exist an open interval of step sizes for which the gradient descent neither converge nor diverge \citep{ma2022multiscale,chen2022gradient}. 
As we are interested in SGD, we complement this result by presenting an example in which loss stabilization occurs almost surely \textit{in the case of stochastic updates}. 
Indeed, consider a regression problem with quadratic parameterization on one-dimensional data inputs $x_i$'s, coming from a distribution $\hat{\rho}$, and outputs generated by the linear model $y_i = x_i\theta_*^2$. The loss writes $F(\theta)  := \frac{1}{4}\mathbb{E}_{\hat{\rho}} \left(y - x\theta^2\right)^2$, and the SGD iterates with step size $\eta>0$ follow, for any $t \in \N$,
\begin{align}\label{eq:SGD}
  \theta_{t+1} &= \theta_{t} + \eta\, \theta_{t}\, x_{i_t} \left( y_{i_t} - x_{i_t}\theta_{t}^2\right) \text{ where } \quad x_{i_t} \sim \hat{\rho}.  
\end{align}
For the sake of clarity, suppose that  $\theta_* = 1$ and $\text{supp}(\hat{\rho}) = [a, b]$, we have the following proposition (a more general result is presented in Proposition~\ref{prop:stabilization-app} of the Appendix).

\begin{proposition}
\label{prop:stabilization}
  For any $\eta \in (a^{-2}, 1.25\cdot b^{-2})$ and initialization $\theta_0 \in (0, 1)$, for all $t > 0$,
  \begin{align}
       &\delta_1 < F(\theta_t) < \delta_2 ~ \text{almost surely, and} \label{eq:level_set} \\
      \hspace*{-.5cm}\exists T>0, \forall k > T,~ &\theta_{t + 2k} < 1 < \theta_{t + 2k+1}~\text{almost surely.} \label{eq:bounce}
  \end{align}
  where $\delta_1, \delta_2, T > 0$ are constant given in the Appendix.
\end{proposition}

The proposition is divided in two parts: \upd{if the step size is large enough}, \eqref{eq:level_set} the loss stabilizes in between level sets $\delta_1$ and $\delta_2$ and \eqref{eq:bounce} shows that after some initial phase, the iterates bounce from one side of the \textit{loss valley} to the other one. Note that despite the stochasticity of the process, the results hold \textit{almost surely}.

\myparagraph{The effective dynamics.} 
As observed in the prototypical SGD training dynamics of Fig.~\ref{fig:fig1_deep_nets_with_wd} and proved in the non-convex toy model of Proposition~\ref{prop:stabilization}, large step sizes lead the loss to stabilize around some level set.
To further understand the effect of this loss stabilization in parameter space, we shall assume perfect stabilization.
Then, from Proposition~\ref{prop:SGD_GD_LN}, \upd{we conjecture the following behaviour}

\vspace{-0.5em}
\begin{center}
\textit{During loss stabilization, SGD is well modelled by GD with constant label noise.}
\end{center}
\vspace{-0.5em}
Label noise dynamics have been studied recently~\citep{blanc2020implicit,damian2021label,li2021happens} thanks to their connection with Stochastic Differential Equations (SDEs). 
To properly write a SDE model, the drift should match the gradient descent and the noise should have the correct covariance structure~\citep{li2019stochastic,wojtowytsch2021stochastic}.  Proposition~\ref{prop:SGD_GD_LN} implies that the noise at state $\theta$ is spanned by the gradient vectors $\{\nabla_\theta h_\theta(x_1), \dots, \nabla_\theta h_\theta(x_n) \}$ and has a constant intensity corresponding to the loss stabilization at a level $\delta>0$. Hence, we propose the following SDE model 
\begin{align}
\label{eq:effective_dynamics_general}
    \mathrm{d} \theta_t = - \nabla_\theta \LL(\theta_t) \mathrm{d} t + \sqrt{\eta \delta}\, \phi_{\theta_t} (X)^\top \mathrm{d} B_t,
\end{align}
where $(B_t)_{t \geq 0}$ is a standard Brownian motion in $\mathbb{R}^n$ and $\phi_\theta (X):=[\nabla_\theta h_\theta(x_i)^\top]_{i = 1}^n \in \R^{n \times p}$ referred to as the \textit{Jacobian} (which is also the \textit{Neural Tangent Kernel (NTK) feature matrix} \citep{jacot2018neural}). This SDE can be seen as \textit{the effective slow dynamics} that drives the iterates while they bounce \textit{rapidly} in some directions at the level set $\delta$. It highlights the combination of the deterministic part of the full-batch gradient and the noise induced by SGD.
\updicml{Beyond the theoretical justification and consistency of this SDE model, we validate it empirically in Sec.~\ref{app:sec:sde_validation} showing that it indeed captures the dynamics of large step size SGD. }
In the next section, we leverage the SDE~(\ref{eq:effective_dynamics_general}) to understand the implicit bias of such learning dynamics.

\subsection{Sparse Feature Learning}
We begin with a simple model of diagonal linear networks that showcase a sparsity inducing dynamics and further disclose our general message about the overall implicit bias promoted by the effective dynamics.

\subsubsection{A warm-up: diagonal linear networks} 
An appealing example of simple non-linear networks that help in forging an intuition for more complicated architectures is diagonal linear networks \citep{vaskevicius2019implicit,woodworth2020kernel,haochen2020shape,pesme2021implicit}. They are two-layer linear networks with only diagonal connections: the prediction function writes $h_{u,v}(x) = \langle u, v \odot x \rangle = \langle u \odot v, x \rangle$ where $\odot$ denotes \textit{elementwise} multiplication. Even though the loss is convex in the associated linear predictor $\beta := u \odot v \in \R^d$, it is not in $(u,v)$, hence the training of such simple models already exhibit a rich non-convex dynamics. In this case, $\nabla_u h_{u,v} (x) = v \odot x$, and the SDE model \eqref{eq:effective_dynamics_general} writes
\begin{align} \label{eq:effective_dynamics}
\mathrm{d} u_t = - \nabla_{u} \LL(u_t,v_t) \,\mathrm{d} t ~ + \sqrt{\eta \delta}\, v_t \odot \left[ X^\top \mathrm{d} B_t \right],
\end{align}
where $(B_t)_{t \geq 0}$ is a standard Brownian motion in $\mathbb{R}^n$. Equations are symmetric for $(v_t)_{t \geq 0}$.

\myparagraph{What is the behaviour of this effective dynamics?} \cite{pillaud2022label} answered this question by analyzing a similar stochastic dynamics and unveiled the sparse nature of the resulting solutions. Indeed, under sparse recovery assumptions, denoting $\beta^*$ the sparsest linear predictor that interpolates the data, it is shown that the associated linear predictor $\beta_t = u_t \odot v_t$: (i) converges exponentially fast to zero outside of the support of $\beta^*$ (ii) is \textit{with high probability} in a $\mathcal{O}(\sqrt{\eta \delta})$ neighborhood of $\beta^*$ in its support after a time $\mathcal{O}(\delta^{-1})$.

\myparagraph{Overall conclusion on the model.} %
During a first phase, SGD with large step sizes $\eta$ decreases the training loss until stabilization at some level set $\delta>0$. 
During this loss stabilization, an effective noise-driven dynamics takes place. It shrinks the coordinates outside of the support of the sparsest signal and oscillates in parameter space at level $\mathcal{O}(\sqrt{\eta \delta})$ on its support. Hence, decreasing later the step size leads to perfect recovery of the sparsest predictor. This behaviour is illustrated in our experiments in Figure~\ref{fig:dln_motivating_picture}.

\subsubsection{The Sparse Feature Learning Conjecture for More General Models} 

Results for diagonal linear nets recalled in the previous paragraph show that the noisy dynamics~(\ref{eq:effective_dynamics}) induce a \textit{sparsity bias}. As emphasized in \citet{haochen2020shape}, this effect is largely due to the multiplicative structure of the noise $ v \odot [X^\top \mathrm{d} B_t]$ that, in this case, has a shrinking effect \textit{on the coordinates} (because of the coordinate-wise multiplication with $v$). In the general case, we see, thanks to \eqref{eq:effective_dynamics_general}, that the same multiplicative structure of the noise still happens but this time \textit{with respect to the Jacobian~$\phi_{\theta} (X)$}. Hence, this suggests that similarly to the diagonal linear network case, the implicit bias of the noise can lead to a shrinkage effect applied to~$\phi_{\theta}(X)$. \upd{ This effect depends on the noise intensity $\delta$ and the step size of SGD}. 
Indeed, an interesting property of Brownian motion is that, for $v \in \R^p$, $\langle v, B_t \rangle = \|v\|_2 W_t$, where the equality holds in law and $(W_t)_{t \geq 0}$ is a one-dimensional Brownian motion. \upd{Hence, the process \eqref{eq:effective_dynamics_general} is equivalent to a process whose $i$-th coordinate is driven by a noise proportional to $\| \phi_i \| \mathrm{d} W^i_t $, where $\phi_i$ is the $i$-th column of $\phi_\theta(X)$ and $(W^i_t)_{t \geq 0}$ is a Brownian motion. This SDE structure, similar to the geometric Brownian motion, is expected to induce the shrinkage of each multiplicative factor~\citep[Section 5.1]{oksendal2013stochastic}, i.e., in our case $(\|\nabla_\theta h(x_i)\|)_{i = 1}^n$. 
Thus, we conjecture:}
\vspace*{-5pt}
\begin{center}
  \textit{The noise part of \eqref{eq:effective_dynamics_general} seeks to minimize the $\ell_2$-norm of the columns of $\phi_{\theta}(X)$.}  
\end{center}
\vspace*{-5pt}
Note that the \textit{fitting part} of the dynamics prevents the Jacobian to collapse totally to zero, but as soon as they are not needed to fit the signal, \textit{columns} can be reduced to zero. Remarkably, from a stability perspective, \citet{blanc2020implicit} showed a similar bias: locally around a minimum, the SGD dynamics implicitly tries to minimize the \textit{Frobenius norm}~$\|\phi_{\theta}(X)\|_{\text{F}} = \sum_{i = 1}^n \|\nabla_\theta h_\theta(x_i)\|^2$.\updicml{ Resolving the above conjecture and characterizing the implicit bias \textit{along} the trajectory of SGD remains an exciting avenue for future work. Now,} we provide a specification of this implicit bias for different architectures:
\begin{itemize}[leftmargin=*, topsep=0pt]
    \item \textbf{Diagonal linear networks:} For $h_{u,v}(x) =  \langle u \odot v, x \rangle$, we have $\nabla_{u,v} h_{u,v} (x) = [v \odot x, u \odot x]$. Thus, for a generic data matrix $X$, minimizing the norm of each column of $\phi_{u,v}(X)$ %
    amounts to put the maximal number of zero coordinates and hence to minimize $\|u \odot v\|_0$. 
    \item \textbf{ReLU networks:} We take the prototypical one hidden layer to exhibit the sparsification effect. Let $h_{a,W}(x) = \langle a, \sigma(Wx)\rangle$, 
     then $\nabla_a h_{a,W}(x) = \sigma(Wx)$ and $\nabla_{w_j} h_{a,W}(x) = a_j x \mathbf{1}_{\langle w_j,x\rangle > 0}$. 
    Note that the $\ell_{2}$-norm of the column corresponding to the neuron is reduced when it is activated at a \textit{minimal number of training points},  hence the implicit bias enables the learning of \textit{sparse data-active features}.  
    Finally, when some directions are needed to fit the data, similarly activated neurons align to fit,\rmv{allowing} \updicml{reducing} the rank of $\phi_{\theta}(X)$. 
\end{itemize}

\updicml{ \myparagraph{Feature sparsity.} Our main insight is that the Jacobian could be significantly simplified during the loss stabilization phase. Indeed, while the gradient part tries to fit the data and align neurons (see e.g. Fig.~\ref{fig:2d_neuron_movement}), the noise part of \eqref{eq:effective_dynamics_general} intends to minimize the $\ell_2$-norm of the columns of $\phi(X)$. Hence, in combination, this motivates us to count the average number of \textit{distinct} (i.e., counting a group of aligned neurons as one), \textit{non-zero} activations over the training set. We refer to this as the \textit{feature sparsity coefficient} {{(see the next section for a detailed description)}}. Note that the aforementioned sparsity comes both in the number of distinct neurons and their activation.}

\rmv{Overall, fully understanding theoretically the structural implications of the implicit bias described above remains an exciting avenue for future work.}We show next that the conjectured sparsity is indeed observed empirically for a variety of models. Remark that both the \updicml{feature sparsity coefficient and the rank} of $\phi_{\theta}(X)$ 
can be used as a good proxy to track the hidden progress during the loss stabilization phase.

\section{Empirical Evidence of Sparse Feature Learning Driven by SGD}
\begin{figure*}[t]
    \centering
    \includegraphics[width=1.0\textwidth]{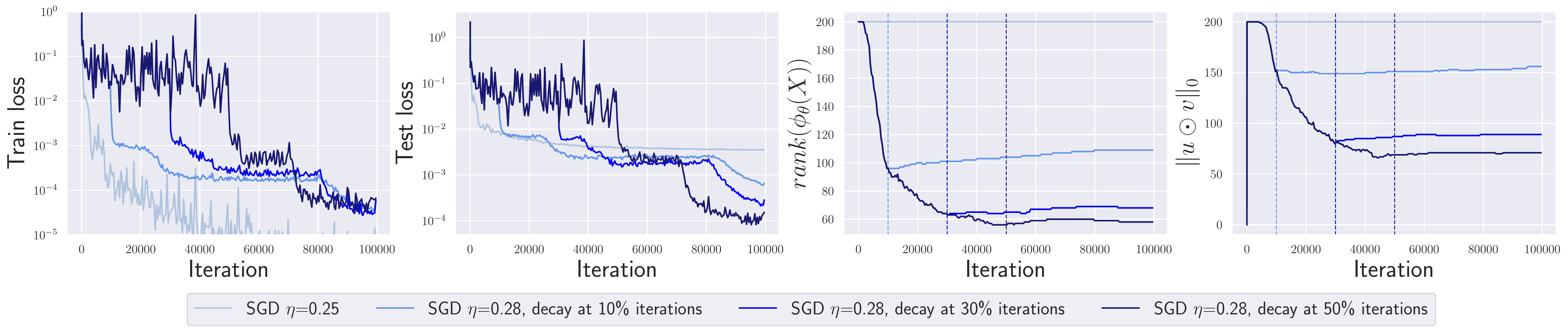}
    \vspace{-8mm}
    \caption{\textbf{Diagonal linear networks}. We observe loss stabilization, better generalization for longer schedules, minimization of the rank of $\phi_\theta(X)$ and sparsity of the predictor $u \odot v$.}
    \vspace{-4mm}
    \label{fig:dln_motivating_picture}
\end{figure*}

Here we present empirical results %
for neural networks of increasing complexity: from diagonal linear networks to deep DenseNets on CIFAR-10, CIFAR-100, and Tiny ImageNet. We make the following common observations for all these networks trained using SGD schedules with large step sizes: %
\begin{enumerate}[label=(\bfseries O\arabic*), itemsep= -1ex, topsep = 0pt]
    \item \textbf{Loss stabilization}: training loss stabilizes around a high level set until step size is decayed,
    \item \textbf{Generalization benefit}: longer loss stabilization leads to better generalization, %
    \item \textbf{Sparse feature learning}: longer loss stabilization leads to sparser features.%
\end{enumerate}
Importantly, \textit{we use no explicit regularization} (in particular, no weight decay) in our experiments so that the training dynamics is driven purely by SGD and the step size schedule. Additionally, in some cases, we cannot find a single large step size that would lead to loss stabilization. In such cases, whenever explicitly mentioned, we use a \textit{warmup} step size schedule---i.e., increasing step sizes according to some schedule---to make sure that the loss stabilizes around some level set. 
Warmup is commonly used in practice \citep{he2016deep, devlin2018bert} and often motivated purely from the optimization perspective as a way to accelerate training \citep{agarwal2021acceleration}, \upd{but we suggest that it is also a way to amplify the regularization effect of the SGD noise which is proportional to the step size.} 

\myparagraph{Measuring sparse feature learning.} \label{par:compute-feature-sparsity} \updicml{We track the simplification of the Jacobian by measuring both the feature sparsity and the rank of $\phi_\theta(X)$. We compute the rank over iterations for each model (except deep networks for which it is prohibitively expensive) by using a fixed threshold on the singular values of $\phi_\theta(X)$ normalized by the largest singular value. In this way, we ensure that the difference in the rank that we detect is not simply due to different scales of $\phi_\theta(X)$. Moreover, we always compute $\phi_\theta(X)$ on the number of fresh samples equal to the number of parameters $|\theta|$ to make sure that rank deficiency is not coming from $n \ll |\theta|$ which is the case in the overparametrized settings we consider. To compute the \textbf{feature sparsity coefficient}, we count the average fraction of \textit{distinct} (i.e., counting a group of highly correlated activations as one), \textit{non-zero} activations at some layer over the training set. 
Note that the value of $100\%$ means a completely \textit{dense} feature vector and $0\%$ means a feature vector with all zeros. 
We count a pair of activations $i$ and $j$ as highly correlated if their Pearson's correlation coefficient 
is at least $0.95$. 
Unlike $\text{rank}(\phi_\theta(X))$, the feature sparsity coefficient scales to deep networks and has an easy-to-grasp meaning.}

\subsection{Sparse Feature Learning in Diagonal Linear Networks}
\label{subsec:experiment_sparse_dynamics_DLN}

\myparagraph{Setup.} The inputs $x_1, \ldots, x_n$ with $n=80$ are sampled from $\mathcal{N}(0,\bm{I}_{d})$ where $\bm{I}_{d}$ is an identity matrix with $d=200$, and the outputs are generated as $y_i = \scal{\beta_*}{x_i}$ where $\beta_* \in \R^{d}$ is $r = 20$ sparse. 
We consider four different SGD runs (started from $u_i=0.1$, $v_i=0$ for each $i$):
one with a small step size and three other with initial large step size decayed after $10\%$, $30\%$, $50\%$ iterations, respectively. 

\myparagraph{Observations.} 
We show the results in Fig.~\ref{fig:dln_motivating_picture} and note that \textbf{(O1)}--\textbf{(O3)} hold even in this simple model trained with vanilla SGD without any explicit regularization or layer normalization schemes. We observe that the training loss stabilizes around $10^{-1.5}$, the test loss improves for longer schedules, both $\text{rank}(\phi_\theta(X))$ and $\|u \odot v\|_0$ decrease during the loss stabilization phase leading to a sparse final predictor. 
While the training loss has seemingly converged to $10^{-1.5}$, a hidden dynamics suggested by \eqref{eq:effective_dynamics} occurs which slowly drifts the iterates to a sparse solution. This implicit sparsification explains the dependence of the final test loss on the time when the large step size is decayed, similarly to what has been observed for deep networks in Fig.~\ref{fig:fig1_deep_nets_with_wd}.
Interestingly, we also note that SGD with large step-size schedules encounters saddle points \textit{after} we decay the step size (see the training loss curves in Fig.~\ref{fig:dln_motivating_picture}) which resembles the saddle-to-saddle regime described in \citet{jacot2021saddle} which does not occur in the large-initialization lazy training regime. 

\myparagraph{SGD and GD have different implicit biases.} 
Since we observe from Fig.~\ref{fig:dln_motivating_picture} that for loss stabilization, stochasticity alone does not suffice and large step sizes are necessary, one may wonder if conversely, only large step sizes can be sufficient to have a sparsifying effect. Even if special instances can be found for which large step sizes are sufficient \upd{(such as for non-centered input features as in \citet{nacson2022implicit})}, we answer this negatively showing that gradient descent in general does not go to the sparsest solution as demonstrated in Fig.~\ref{fig:dln_motivating_picture_gd} in the Appendix. 
\setlength{\columnsep}{2.5mm}
\begin{wrapfigure}{r}{0.24\textwidth}
    \vspace{-5mm}
    \centering
    \includegraphics[width=0.26\textwidth]{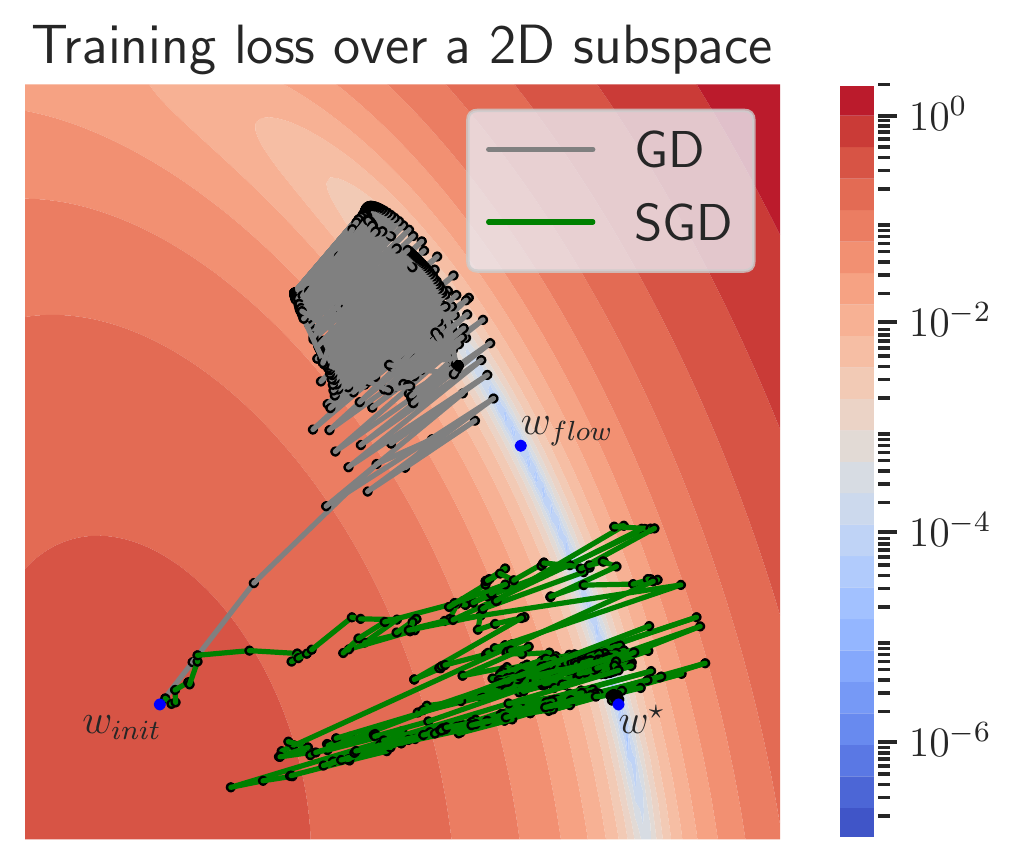}
    \vspace{-8mm}
    \caption{\textbf{Diagonal linear networks.} GD and SGD take different trajectories.}
    \vspace{-5mm}
    \label{fig:gd_vs_sgd_in_2d}
\end{wrapfigure}
Moreover, in Fig.~\ref{fig:gd_vs_sgd_in_2d}, we visualize the difference in trajectory between the two methods taken with large step sizes over a 2D subspace spanned by $w^\star - w_{init}$ and $w_{flow} - w_{init}$, where $w^\star$ is the ground truth, $w_{flow}$ is the result of gradient flow, and $w_{init}$ is the initialization. This example provides an intuition that loss stabilization alone is not sufficient for sparsification and that the role of noise described earlier is crucial.

\subsection{Sparse Feature Learning in Simple ReLU Networks}
\begin{figure*}[t]
    \centering
    \includegraphics[width=0.99\textwidth]{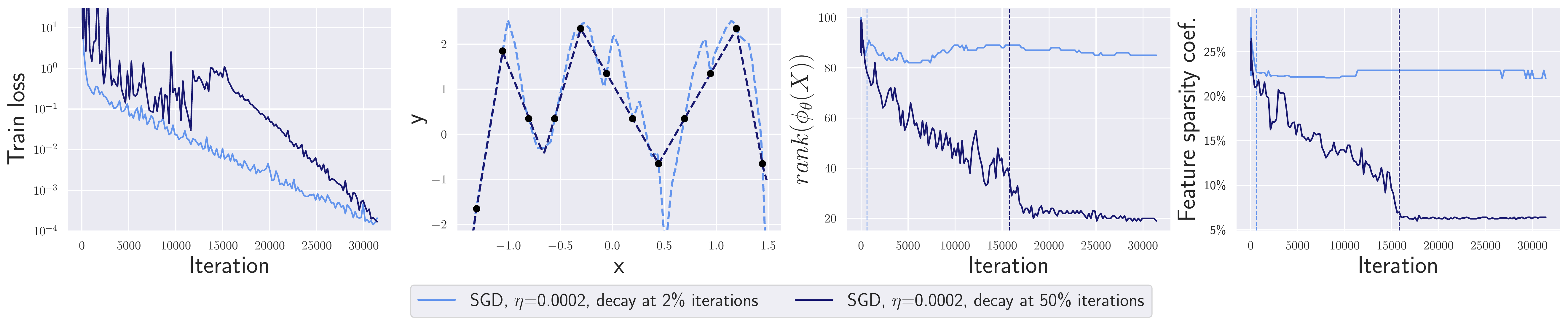}
    \vspace{-3mm}
    \caption{\textbf{Two-layer ReLU networks for 1D regression}. %
    We observe loss stabilization, simplification of the model trained with a longer schedule, lower rank of $\phi_\theta(X)$, and much sparser features.}
    \vspace{-2mm}
    \label{fig:fc_net_1d_regression_main_plots}
\end{figure*}

\begin{figure*}[t]
    \centering
    \includegraphics[width=0.99\textwidth]{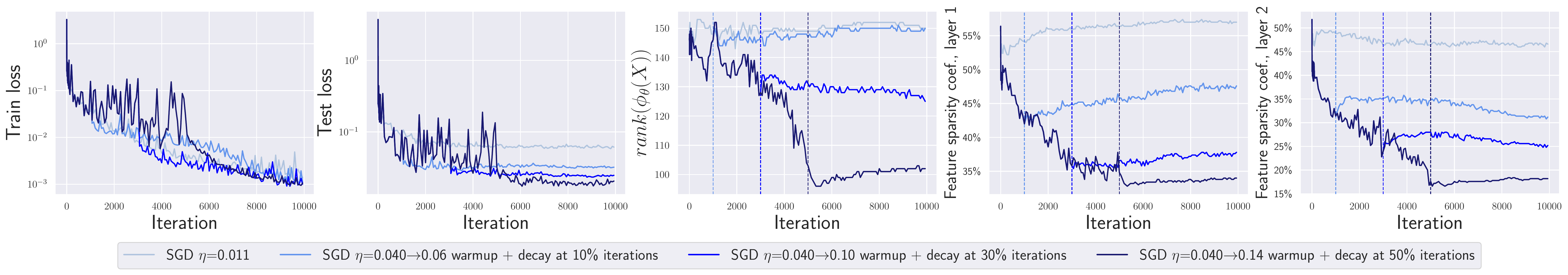}
    \vspace{-3mm}
    \caption{\textbf{Three-layer ReLU networks in a teacher-student setup}. We observe loss stabilization, lower rank of the Jacobian and lower feature sparsity \upd{coefficient} on \textit{both} hidden layers.}
    \vspace{-2mm}
    \label{fig:fc_net_three_layer_teacher_student}
\end{figure*}

\myparagraph{Two-layer ReLU network in 1D.}
We consider the 1D regression task from \citet{blanc2020implicit} with 12 points, where label noise SGD has been shown to learn a simple model. 
We show that similar results can be achieved with large-step-size SGD via loss stabilization. 
We train a ReLU network with $100$ neurons with SGD with a long linear warmup (otherwise, we were unable to achieve approximate loss stabilization), directly followed by a step size decay. The two plots correspond to a warmup/decay transition at $2\%$ and $50\%$ of iterations, respectively. 
The results shown in Fig.~\ref{fig:fc_net_1d_regression_main_plots} confirm that \textbf{(O1)}--\textbf{(O3)} hold: the training loss stabilizes around $10^{-0.5}$, the predictor becomes much simpler and is expected to generalize better, and both $\text{rank}(\phi_\theta(X))$ and the feature sparsity coefficient substantially decrease during the loss stabilization phase.
Interestingly, the rank reduction of $\phi_\theta(X)$ occurs because of zero \textit{activations}, and not because of zero \textit{weights}. 
For this one-dimensional task, we can directly observe the final predictor which is sparse in terms of the number of distinct ReLU kinks (i.e., having a few piecewise-linear segments) as captured by the feature sparsity coefficient and the rank of the Jacobian.
Interestingly, we also observed \textit{overregularization} for even larger step sizes when we cannot fit all the training points (see Fig.~\ref{fig:fc_net_1d_regression_overreg} in Appendix). This phenomenon clearly illustrates how the capacity control is induced by the optimization algorithm: \textit{the function class over which we optimize depends on the step size schedule}. 
Additionally, Fig.~\ref{fig:fc_net_1d_regression_classifier_vis} in App. shows the evolution of the predictor over iterations. The general picture is confirmed: first the model is simplified during the loss stabilization phase and only then fits the data.

\myparagraph{Deeper ReLU networks.}
We use a teacher-student setup with a random \textit{three-layer} teacher ReLU network having $2$ neurons on each hidden layer. The student network is overparametrized with $10$ neurons on each layer and is trained on $50$ examples. \upd{Such teacher-student setup is useful since we know that the student network can implement the ground truth function but might not find it due to the small sample size.} We train models using SGD with a medium constant step size and a large step size with warmup decayed after $10\%$, $30\%$, $50\%$ iterations, respectively.
The results shown in Fig.~\ref{fig:fc_net_three_layer_teacher_student} confirm that \textbf{(O1)}--\textbf{(O3)} hold: the training loss stabilizes around $10^{-1.5}$, the test loss is smaller for longer schedules, and both $\text{rank}(\phi_\theta(X))$ and the feature sparsity coefficient substantially decrease during the loss stabilization phase. All methods have the same value of the training loss ($10^{-3}$) after $10^4$ iterations but different generalization. Moreover, we see that the feature sparsity coefficient decreases \textit{on each layer} which makes this metric a promising one to consider for  deeper networks.

\subsection{Sparse Feature Learning in Deep ReLU Networks}

\begin{figure*}[h!] \centering \small
    \begin{tabular}{c} 
    \textbf{DenseNet-100 on CIFAR-10, basic setting (no momentum and augmentations)} \\
    \includegraphics[width=0.99\textwidth]{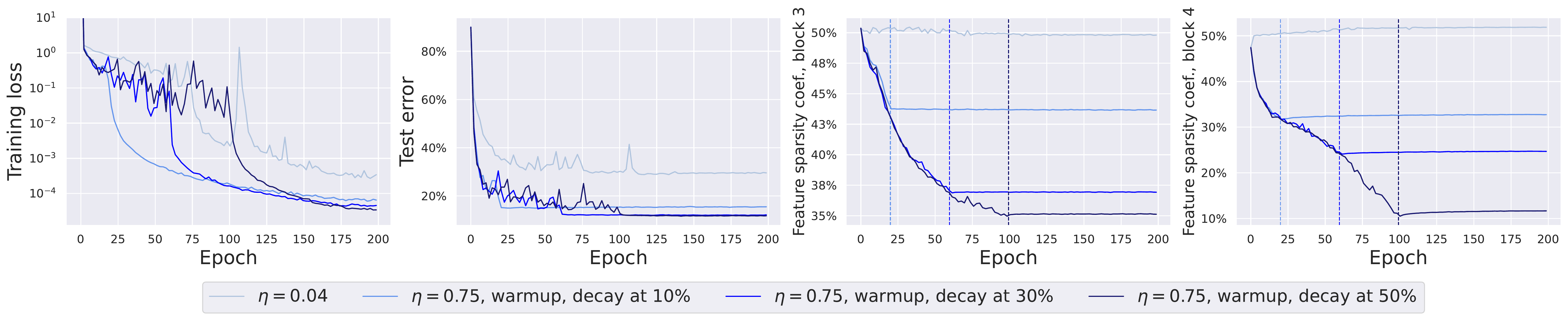}\\
    \textbf{DenseNet-100 on CIFAR-100, basic setting (no momentum and augmentations)} \\
    \includegraphics[width=0.99\textwidth]{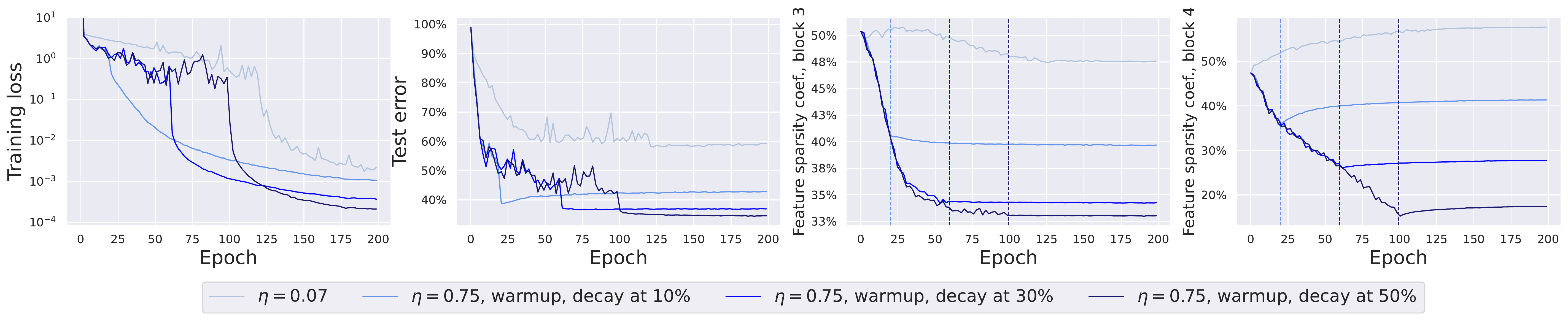}\\
    \textbf{DenseNet-100 on Tiny ImageNet, basic setting (no momentum and augmentations)} \\
    \includegraphics[width=0.99\textwidth]{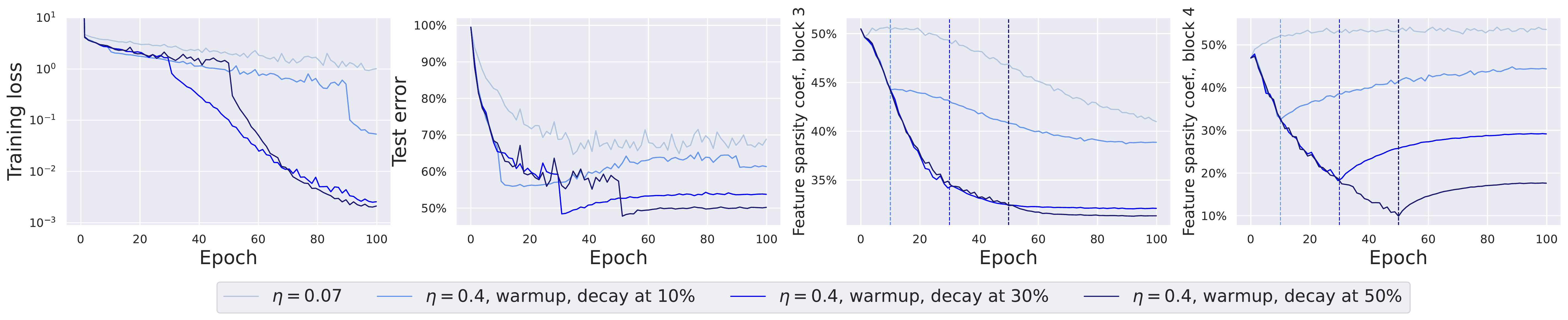}
    \end{tabular}
    \vspace{-4mm}
    \caption{\textbf{Experiments with DenseNet-100 in the basic setting.} We can see that the training loss stabilizes, the test error noticeably depends on the length of the schedule, and the feature sparsity coefficient is minimized during the large step size phase.}
    \label{fig:densenet_minimal}
\end{figure*}

\begin{figure*}[h!] \centering \small
    \begin{tabular}{c} 
    \textbf{DenseNet-100 on CIFAR-10, state-of-the-art setting (with momentum and augmentations)} \\
    \includegraphics[width=0.99\textwidth]{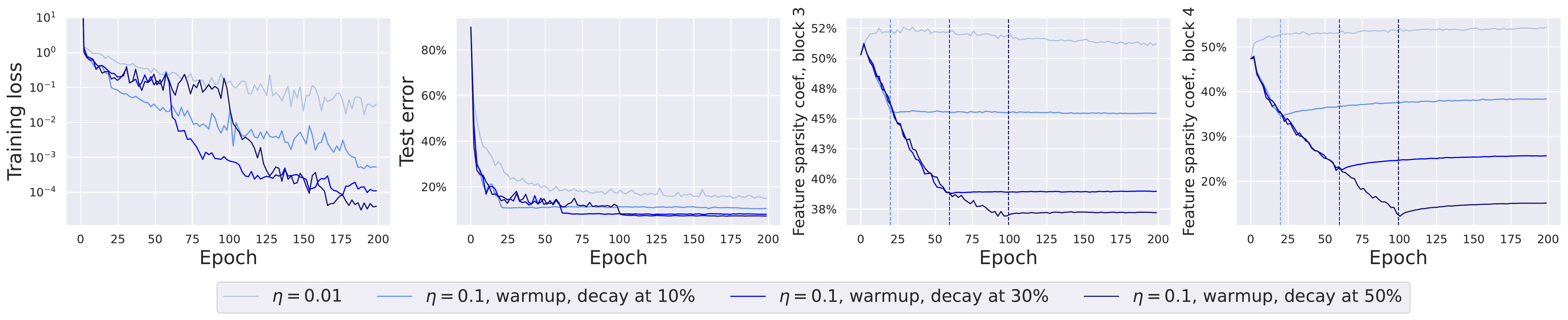}\\
    \textbf{DenseNet-100 on CIFAR-100, state-of-the-art setting (with momentum and augmentations)} \\
    \includegraphics[width=0.99\textwidth]{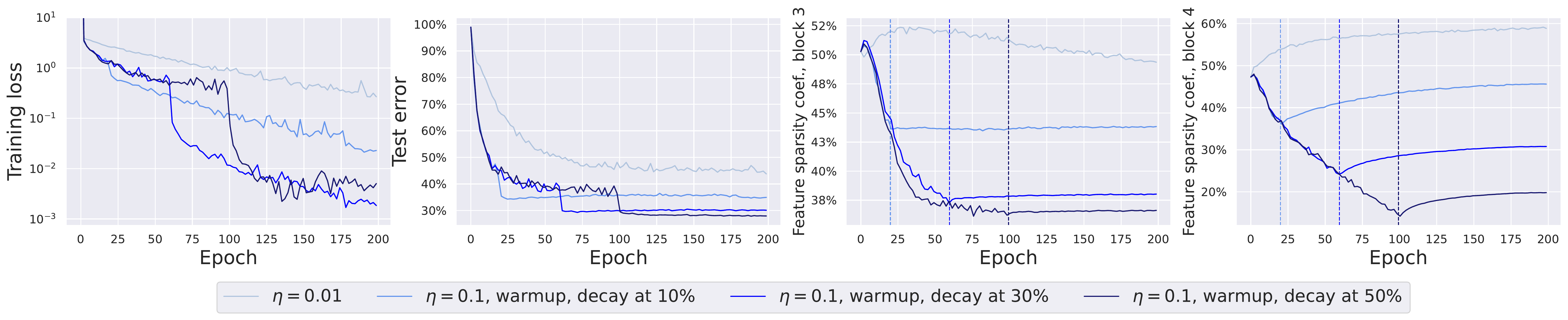}\\
    \textbf{DenseNet-100 on Tiny ImageNet, state-of-the-art setting (with momentum and augmentations)} \\
    \includegraphics[width=0.99\textwidth]{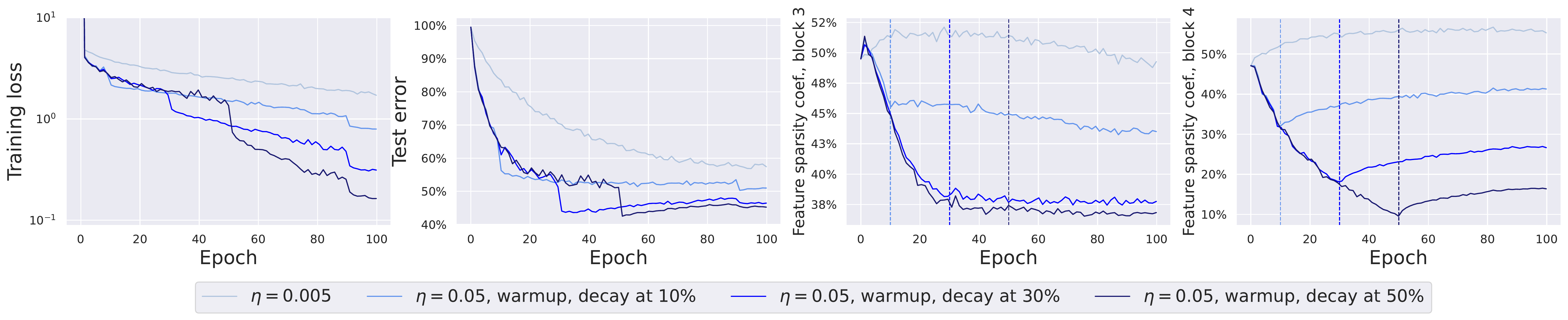}
    \end{tabular}
    \vspace{-4mm}
    \caption{\textbf{Experiments with DenseNet-100 in the state-of-the-art setting.} We can see that the training loss stabilizes, the test error noticeably depends on the length of the schedule, and the feature sparsity coefficient is minimized during the large step size phase.}
    \label{fig:densenet_sota}
\end{figure*}

\myparagraph{Setup.}
We consider here an image classification task and train a DenseNet-100-12 on CIFAR-10, CIFAR-100, and Tiny ImageNet using SGD with batch size $256$ and different step size schedules. We use an exponentially increasing warmup schedule with exponent $1.05$ to stabilize the training loss. We cannot measure the rank of $\phi(X)$ here since this matrix is too large ($\approx (5\tiny{\times}10^{4} ) \times (2\tiny{\times}10^{7})$) so we measure only the feature sparsity coefficient taken at two layers: at the end of super-block 3 (i.e., in the middle of the network) and super-block 4 (i.e., right before global average pooling at the end of the network) of DenseNets. We test two settings: a basic setting and a state-of-the-art setting with momentum and standard augmentations.

\myparagraph{Observations.} The results shown in Fig.~\ref{fig:densenet_minimal}~and~\ref{fig:densenet_sota} confirm that our main findings also hold for deep convolutional networks used in practice: 
the training loss approximately stabilizes, the test error is becoming progressively better for longer schedules, and the feature sparsity coefficient gradually decreases at both super blocks 3 and 4 until the step size is decayed. 
We also see that small step sizes consistently lead to suboptimal generalization, e.g., $60\%$ vs. $35\%$ in the basic setting on CIFAR-100.  
This poor performance confirms that it is crucial to leverage the implicit bias of large step sizes. 
The difference in the feature sparsity coefficient is also substantial: typically $50\%$-$60\%$ for small step sizes vs. $10\%$-$20\%$ for larger step sizes at block 4. %
The observations are similar for the state-of-the-art setting as well where we also see a noticeable difference in generalization and feature sparsity depending on the step size and schedule. 
Finally, we note that while both the feature sparsity coefficient and test error decrease together, it remains to be seen whether they are causally related on natural datasets. 

We show the results with similar findings for other architectures (ResNets-18 and ResNets-34) on CIFAR-10 and CIFAR-100 in Fig.~\ref{fig:c10_deep_nets} 
 and Fig.~\ref{fig:c100_deep_nets} in Appendix. 
Additionally, Fig.~\ref{fig:feature_learning_conv_filters_cifar10} illustrates that for small step sizes, the early and middle layers stay very close to their random initialization which indicates the absence of feature learning similarly to what is suggested by the neuron movement plot in Fig.~\ref{fig:2d_neuron_movement} in the Appendix for a two-layer network in a teacher-student setup.

\section{Conclusions and Insights from our Understanding of the Training Dynamics}
Here we provide an extended discussion on the training dynamics of neural networks resulting from our theoretical and empirical findings. 

 \myparagraph{The multiple stages of the SGD training dynamics.} 
As analyzed and shown empirically, the training dynamics we considered can be split onto three distinct phases: (i) an initial phase of reducing the loss down to some level where stabilization can occur, (ii) a loss stabilization phase where noise and gradient directions combine to find architecture-dependent sparse representations of the data, (iii) a final phase when the step size is decreased to fit the training data. This typology clearly disentangles the effect of the stabilization phase (ii) which relies on the implicit bias of SGD to simplify the model.
Note that phases (ii) and (iii) can be repeated until final convergence \citep{he2016deep}.
Moreover, in some training schedules, (ii) does not explicitly occur, and the effect of loss stabilization (ii) and data fitting (iii) can occur simultaneously 
\citep{loshchilov2018decoupled}.

\myparagraph{From lazy training to feature learning.}
Similar sparse implicit biases have been shown for regression with  infinitely small initialization~\citep{boursier2022gradient} and for %
classification \citep{chizat2020implicit, Lyu2020Gradient}. However, both approaches are not practical from the computational point of view since (i) the origin is a saddle point for regression leading to the vanishing gradient problem (especially, for deep networks), and (ii) max-margin bias for classification is only expected to happen in the asymptotic phase~\citep{moroshko2020implicit}.
On the contrary, large step sizes enable to initialize far from the origin, while allowing to \textit{efficiently} transition from a regime close to the lazy NTK regime \citep{jacot2018neural} to the rich feature learning regime. %

\myparagraph{Common patterns in the existing techniques.}
Tuning the step size to obtain loss stabilization can be difficult.
To prevent early divergence caused by too large step sizes, we sometimes had to rely on an increasing step size schedule (known as \textit{warmup}). 
Interpreting such schedules as a tool to favor implicit regularization provides a new explanation to their success and popularity.
Additionally, normalization schemes like \textit{batch normalization} or \textit{weight decay}, beyond carrying their own implicit or explicit regularization properties, can be analyzed from a similar lens: they allow to use larger step sizes that boost further the implicit bias effect of SGD while preventing divergence~\citep{bjorck2018understanding,zhang2018three,li2019exponential}. 
Note also that we derived our analysis with batch size equal to one for the sake of clarity, but an arbitrary batch size $B$ would simply be equivalent to replacing~$\gamma \leftarrow \gamma / B$. Similarly to the consequence of large step sizes, preferring \textit{smaller batch sizes} \citep{keskar2016large} while avoiding divergence seem key to benefit from the implicit bias of SGD. 
Finally, the effect of large step sizes or small batches is often connected to measures of \textit{flatness} of the loss surface via stability analysis \citep{wu2018sgd} and some methods like the Hessian regularization \citep{damian2021label} or SAM \citep{foret2021sharpnessaware} explicitly optimize it. 
Such methods resemble the implicit bias of SGD with loss stabilization implied by the label noise equation~(\eqref{eq:effective_dynamics_general}) where matrix $\phi_\theta(X)$ is the key component of the Hessian. %
However, an important practical difference is that the regularization strength in these methods is explicit and decoupled from the step size schedule which may be harder to properly tune since it is simultaneously responsible for optimization \textit{and} generalization.

\section*{Acknowledgements}
We thank Maria-Luiza Vladarean and Scott Pesme for insightful discussions and proofs checking.
M.A. was supported by the Google Fellowship and Open Phil AI Fellowship. 
A.V. was supported by the Swiss Data Science Center Fellowship. 
L.P.V.'s work was partially supported by the Simons Foundation.

\bibliography{literature}
\bibliographystyle{icml2023}

\newpage
\appendix
\onecolumn

{\LARGE \noindent\bfseries  \textsc{Appendix}}

\vspace{0.5cm}

\noindent In Section~\ref{app:sec:sgd-label-noise}, we show Proposition~\ref{prop:SGD_GD_LN} on the equivalence between SGD and GD with added noise. In Section~\ref{app:sec:stabilization}, we provide the proof that loss stabilization occurs as written in Proposition~\ref{prop:stabilization}. \upd{In Section \ref{app:sec:sde_validation}, we show experimentally that the proposed SDE model matches well the SDE dynamics}. Finally, we present additional experiments in Section \ref{app:sec:additional_experiments}.
\vspace{-0.25cm}
\begin{figure}[h]
    \centering
    \includegraphics[width=0.6\textwidth]{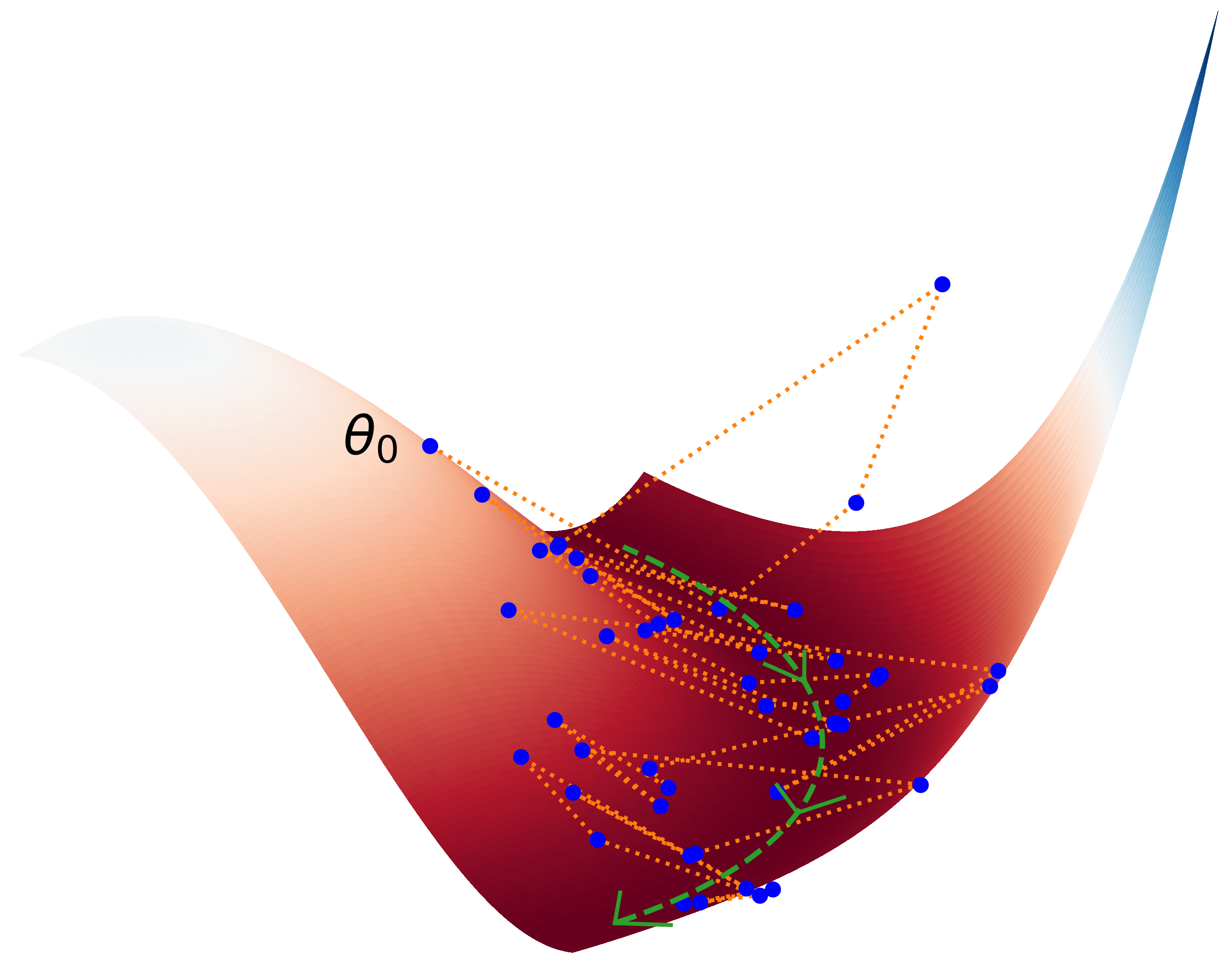}\\
    \caption{\upd{Three-dimensional visualisation of the SGD dynamics in a non-convex loss landscape. The SGD dynamics (blue points) is bouncing side-to-side to the bottom of the valley (the dotted green line). A slow movement occurs pushing the iterates in the direction given by the green arrows.}}
    \label{fig:loss_lanscape}
\end{figure}

\upd{To begin this appendix, we provide in Figure~\ref{fig:loss_lanscape} a toy visualization in which we showcase a typical SGD dynamics when loss stabilization occurs.  We run SGD on the diagonal linear network with one sample in two dimensions ($n=1, d=2$) adding label noise of the shape given by equation \eqref{eq:effective_dynamics}, with balanced layers $u = v$. The blue points corresponds to iterates of the dynamics (that are linked with the orange dotted lines). The green line corresponds to the global minimum of the loss, what can be called the ``bottom of the valley''. This hopefully will serve the reader forge a visual intuition on (i) the bouncing dynamics side-to-side to the bottom of the valley (in green), and (ii) the slow stochastic movement (in the direction of the green arrows).}

\label{sec:appendix}

\section{SGD and Label Noise GD}
\label{app:sec:sgd-label-noise}

For the sake of clarity we recall below the statement of the Proposition~\ref{prop:SGD_GD_LN} which we prove in this section.

\begin{customproposition}{\ref{prop:SGD_GD_LN}}
Let $(\theta_t)_{t \geq 0}$ follow the SGD dynamics \eqref{eq:general_SGD} with sampling function $(i_t)_{t \geq 0}$. Let $\mathbf{1}_{i = i_t}$ be indicator function, define for $t \geq 0$, the random vector $\xi_t \in \R^n$ such that for all $i \in \llbracket 1,n \rrbracket$,
\begin{align}
[\xi_{t}]_i:= ( h_{\theta_t} (x_i) - y_i) (1 - n \mathbf{1}_{i = i_t}).
\end{align}
\upd{Then $(\theta_t)_{t \geq 0}$ follows the full-batch gradient dynamics on $\LL$ with label noise $(\xi_{t})_{t \geq 0}$, that is
\begin{align}
\theta_{t+1} = \theta_t -\frac{\eta}{n} \sum_{i=1}^n (h_{\theta_t}(x_i) - y^t_{i}) \nabla_\theta h_{\theta_t}(x_i),
\end{align}
where we define the random labels $y^t := y + \xi_{t}$. Furthermore, $\xi_t$ is a mean zero random vector with variance such that $\frac{1}{n(n-1)}\E \left\| \xi_t \right\|^2= 2 \LL(\theta_t)$.}
\end{customproposition}

\begin{proof}
Note that 
\eqals{ \sum_{i=1}^n (h_{\theta_t}(x_i) - y^t_{i}) \nabla_\theta h_{\theta_t}(x_i) &=  \sum_{i=1}^n (h_{\theta_t}(x_i) - y_i - [\xi_{t}]_i ) \nabla_\theta h_{\theta_t}(x_i).}
Using $[\xi_{t}]_i:= ( h_{\theta_t} (x_i) - y_i) (1 - n \mathbf{1}_{i = i_t})$, 
\eqals{ &= \frac{1}{n} \sum_{i=1}^{n} (h_{\theta_t}(x_{i}) - y_{i} - ( h_{\theta_t} (x_i) - y_i) (1 - n \mathbf{1}_{i = i_t}) ) \nabla_\theta h_{\theta_t}(x_{i}), \\
&= \sum_{i=1}^{n} \mathbf{1}_{i = i_t}  (h_{\theta_t}(x_{i}) - y_{i} ) \nabla_\theta h_{\theta_t}(x_{i}) = (h_{\theta_t}(x_{i_t}) - y_{i_t} ) \nabla_\theta h_{\theta_t}(x_{i_t}).
} which is exactly the stochastic gradient wrt to sample $(x_{i_t},y_{i_t})$.

Now we prove the latter part of the proposition regarding the scale of the noise. Recall that, for all $i\leqslant n$, we have $[\xi_{t}]_i = ( h_{\theta_t} (x_i) - y_i) (1 - n \mathbf{1}_{i = i_t})$, where $ i_t \sim  \mathcal{U}\left( \llbracket 1,n \rrbracket \right)$. Now taking the expectation, 
\eqals{ \mathbb{E}[\xi_{t}]_i = \Ex{ ( h_{\theta_t} (x_i) - y_i) (1 - n \mathbf{1}_{i = i_t}) } = ( h_{\theta_t} (x_i) - y_i) (1 - n \Ex{\mathbf{1}_{i = i_t}} ) = 0, 
}
as $ \Ex{\mathbf{1}_{i = i_t}} = 1/n$. Coming to the variance, 
\eqals{    \E \left\| \xi_t \right\|^2 &= \Ex{ \sum_{i=1}^{n} {[\xi_{t}]_i}^2 } = \sum_{i=1}^{n}  \E {[\xi_{t}]_i}^2  \\
    &=  \sum_{i=1}^{n} ( h_{\theta_t} (x_i) - y_i)^2 \Ex{(1 - n \mathbf{1}_{i = i_t})^2 } \\
    &= \sum_{i=1}^{n} ( h_{\theta_t} (x_i) - y_i)^2 \Ex{(1 - 2 n \mathbf{1}_{i = i_t} + n^2 \mathbf{1}_{i = i_t} ) } \\
    &= \sum_{i=1}^{n} ( h_{\theta_t} (x_i) - y_i)^2 ( 1 - 2 + n ) \\
    &= (n-1)  \sum_{i=1}^{n} ( h_{\theta_t} (x_i) - y_i)^2 = 2n(n-1) \LL(\theta_{t}),
}
and this concludes the proof of the proposition.
\end{proof}

\section{Quadratic Parameterization in One Dimension} 
\label{app:sec:stabilization}

Again, for the Appendix to be self-contained, we recall the setup of the Proposition~\ref{prop:stabilization} on loss stabilization. We consider a regression problem with quadratic parameterization on one-dimensional data inputs $x_i$'s, coming from a distribution $\hat{\rho}$, and outputs generated by the linear model $y_i = x_i\theta_*^2$. The loss writes $F(\theta)  := \frac{1}{4}\mathbb{E}_{\hat{\rho}} \left(y - x\theta^2\right)^2$, and the SGD iterates with step size $\eta>0$ follow, for any $t \in \N$,
\begin{align}\label{eq:SGD_app}
  \theta_{t+1} &= \theta_{t} + \eta\, \theta_{t}\, x_{i_t} \left( y_{i_t} - x_{i_t}\theta_{t}^2\right) \qquad \text{ where } \quad x_{i_t} \sim \hat{\rho}.  
\end{align}
We rewrite the proposition here.

\begin{proposition} \normalfont{\textbf{(Extended version of Proposition~\ref{prop:stabilization})}}
\label{prop:stabilization-app}
  Assume $\exists\ x_{\min}, x_{\max} >0$ such that $\text{supp}(\hat{\rho}) \subset [x_{\min}, x_{\max}]$. Then for any $\eta \in ((\theta_* x_{\min})^{-2}, 1.25 (\theta_* x_{\max})^{-2} )$, any initialization in $\theta_0 \in (0, \theta_*)$, for $t \in \N$, we have almost surely
  \begin{align}
  \label{eq:prop_lower_bound_bouncing}
      F(\theta_t) \in \left( \epsilon_{\mathrm{o}}^2 ~ \theta_*^2 , 0.17 ~\theta_*^2 \right)  .
  \end{align}
  where $\epsilon_{\mathrm{o}} = \min \left\lbrace ( \eta (  \theta_* x_{\min})^{2} - 1 )/3, 0.02  \right\rbrace $. Also, almost surely, there exists $t,k > 0$ such that $  \theta_{t+2k} \in \left( 0.65~\theta_*, (1- \epsilon_{\mathrm{o}}) ~ \theta_* \right) $  and $\theta_{t+2k+1} \in \left( (1+ \epsilon_{\mathrm{o}})~\theta_* ,1.162 ~\theta_* \right)$.  
\end{proposition}

\begin{proof}
\noindent Consider SGD recursion \eqref{eq:SGD_app} and note that $y = x \theta_*^2$. 
\eqals{  \theta_{t+1} = \theta_{t} + \eta \, \theta_{t} \, x( x\theta_{*}^2 - x\theta_{t}^2 )  \\
  \theta_{t+1} = \theta_{t} + \eta \, \theta_{t} \, x^2\,( \theta_{*}^2 - \theta_{t}^2 )
} For the clarity of exposition, we consider the rescaled recursion of the original SGD recursion. 
\eqals{  \nicefrac{\theta_{t+1}}{\theta_{*}} = \nicefrac{\theta_t}{\theta_{*}} + \eta~\theta_{*}^2~x^2~\nicefrac{\theta_t}{\theta_{*}}\left( 1 - \left(\nicefrac{\theta_t}{\theta_{*}}\right)^2 \right),
}
and, by making the benign change $\theta_t \leftarrow \theta_t/\theta_*$, we focus on the stochastic recursion instead,
\begin{align} \label{eq:rec-1D-main}
\theta_{t+1} = \theta_{t}  + \gamma \theta_{t} ( 1 - \theta_{t}^2), 
\end{align}
where $\gamma \sim \hat{\rho}_{\gamma}$ the pushforward of $\hat{\rho}$ under the application $z \to \eta\, \theta_*^2\, z^2 $. Let $\Gamma:=\text{supp}(\hat{\rho}_{\gamma})$, the support of the distribution of $\gamma$. From the range of $\eta$, it can be verified that $\Gamma \subseteq (1,1.25) $. Now the proof of the theorem follows from Lemma~$\ref{lem:oscillate}$.
\end{proof}

\begin{lemma}[Bounded Region] \label{lem:bounded}
  Consider the recursion~\eqref{eq:rec-1D-main}, for $\Gamma \subseteq (1,1.25)$ and  $0<\theta_{0}<1$, then for all $ t > 0 $, $\theta_{t} \in \left( 0,1.162 \right) $.
\end{lemma}
\begin{proof}

Consider a single step of \eqref{eq:rec-1D-main}, for some $\gamma \in (1,1.25)$, 
\begin{align*}
  \theta_{+} = \theta + \gamma \theta ( 1 - \theta^2)
\end{align*}
The aim is to show that $\theta_{+}$ stays in the interval $(0,1.162)$. In order to show this, we do a casewise analysis. 

\noindent\underline{For $\theta \in \left(0,1\right]$}:
Since $0 < \theta \leq 1$, we have $\theta_{+} \geq \theta > 0$. To prove the bound above, consider the following quantity, 
\begin{align}
  \theta_{max} = \underset{\gamma \in (1,1.25)}{\max}~\underset{\theta \in (0,1]}{\max} ~\theta + \gamma \theta ( 1 - \theta^2)
\end{align}
Say $h_{\gamma}(\theta) = \theta + \gamma \theta ( 1 - \theta^2)$, note that $h'_{\gamma}(\theta) = 1 + \gamma - 3 \gamma \theta^2$ and $h''_{\gamma}(\theta) = - 6 \gamma \theta < 0$. Hence, for any $\gamma$ in our domain, the maximum is attained at $ \theta_{\gamma} = \frac{1}{\sqrt{3}} \sqrt{\frac{1}{\gamma} + 1} $ and $h_{\gamma}(\theta_{\gamma}) = \frac{2(1+\gamma)^{3/2}}{3\sqrt{3\gamma}}$. 

\eqals{\underset{\gamma \in (1,1.25)}{\max}~\underset{\theta \in (0,1]}{\max}~\theta + \gamma \theta ( 1 - \theta^2) = \underset{\gamma \in (.5,1.25)}{\max} \frac{2(1+\gamma)^{3/2}}{3\sqrt{3\gamma}} } It can be verified that $\frac{2(1+\gamma)^{3/2}}{3\sqrt{3\gamma}}$ is increasing with gamma in the interval $(1,1.25)$. Hence,
\begin{align}
  \underset{\gamma \in (1,1.25)}{\max} \frac{2(1+\gamma)^{3/2}}{3\sqrt{3\gamma}} \leq \left.\frac{2(1+\gamma)^{3/2}}{3\sqrt{3\gamma}}\right\vert_{\gamma = 1.25} < 1.162
\end{align}

Combining them, we get, 
\begin{align}
  \theta_{+} \leq \underset{\gamma \in (0,1.25)}{\max}~\underset{\theta \in (0,1]}{\max}~\theta + \gamma \theta ( 1 - \theta^2) < 1.162
\end{align}

\noindent\underline{For $\theta \in \left(1,1.162\right)$}: 
Since $\theta > 1$, we have, $\theta_{+} < \theta < 1.162$. For lower bound, note that for $\theta_{+}$ to be less than $0$, we need  $ 1 + \gamma - \gamma \theta^2 < 0 $. But for $\gamma \in (1,1.25)$ and $\theta \in (1,1.162) $,
\eqals{  \gamma (\theta^2 - 1) < 1.25 ( (1.162)^2 - 1) < 1. }
Hence, it never goes below $0$.
\end{proof}

\begin{lemma}
\label{lem:oscillate}
  Consider the recursion~\eqref{eq:rec-1D-main} with $\Gamma \subseteq (1,1.25)$ and $\theta_{0}$ initialized uniformly in $(0,1)$. Then, there exists $ \epsilon_0 > 0$, such that for all $ \epsilon < \epsilon_0$ there exists $t>0$ such that for any $k > 0$, 
\eqals{
\theta_{t+2k} \in (0.65,1-\epsilon) \quad and \quad \theta_{t+2k+1} \in (1+\epsilon,1.162)  
}  almost surely.
\end{lemma}

\begin{proof}
    Define $\gamma_{\mathrm{min}} > 1$ as the infimum of the support $\Gamma$. Let $\epsilon_{o} = \min \lbrace \nicefrac{(\gamma_{min} - 1) }{3}, 0.02 \rbrace$. Note that $\epsilon_0>0$ as $\gamma_{\mathrm{min}} > 1$. Now for any $0<\epsilon < \epsilon_{o} $, we have $\gamma_{\mathrm{min}}(2-\epsilon)(1-\epsilon) > 2$. 

Divide the interval (0,1.162) into 4 regions, $\itvl{0} = (0,0.65]$, $~\itvl{1} = (0.65,1-\epsilon)$, $~\itvl{2} = [1-\epsilon,1)$, $~\itvl{3} = (1,1.162) $.  The strategy of the proof is that the iterates will eventually end up in $\itvl{1}$ and that once it ends up in $\itvl{1}$, it comes back to $\itvl{1}$ in 2 steps.  

Let $\theta_0$  be initialized uniformly random in $(0,1)$. Consider the sequence $(\theta_{t})_{t \geq 0}$ generated by  
\begin{align}
    \theta_{t+1} = \hgam{\gamma_{t}}(\theta_{t})  :=  \theta_{t} + \gamma_{t} \theta_{t} 
(1 - \theta_{t}^2) \quad  \text{where} \quad \gamma_{t} \sim \hat{\rho}_{\gamma}.  
  \end{align}
  We prove the following facts {\bfseries (P1)-(P4)}:
  \begin{description}
    \item[(P1)] There exists $t \geq 0$ such that the $\theta_{t} \in \itvl{1} \cup \itvl{2} \cup \itvl{3}$. 
    \item[(P2)] Let $\theta_{t} \in \itvl{3}$, then $\theta_{t+1} \in \itvl{1} \cup \itvl{2}$. 
    \item[(P3)] Let $\theta_{t} \in \itvl{2}$, there exists $k>0$ such that for $k'<k$, $\theta_{t+2k'} \in \itvl{2}$ and $\theta_{t+2k} \in \itvl{1}$. 
    \item[(P4)] When $\theta_{t} \in \itvl{1}$, then for all $k\geq0$, $\theta_{t+2k} \in \itvl{1}$ and $\theta_{t+2k+1} \in (1+\epsilon,1.162)$.   
  \end{description}

  \noindent \underline{\textbf{Proof of (P1)-(P4)}:} Let $t \in \N$, note first that the event $\{\theta_t = 1\} = \cup_{k \leqslant t}\{\theta_k = 1 | \theta_{k-1} \neq 1\}$ and hence a finite union of zero measure sets. Hence $\{\theta_t = 1\}$ is a zero measure set and therefore we do not consider it below. For any other sequence, from the above four properties, we can conclude that the lemma holds.

  \noindent \underline{\textbf{Proof of P1}:}  Assume that until time $t>0$, the iterates are all in $\itvl{0}$, then we have 
\eqals{  \theta_{t}  = \theta_{t-1} ( 1 + \gamma (1-\theta_{t-1}^2) ) \geq \theta_{t-1} (2 - \theta_{t-1}^2) > 1.5 ~\theta_{t-1} > 1.5^t ~\theta_{0} } 
  Hence, the sequence eventually exits $\itvl{0}$. We know that it will stay bounded from Lemma~\ref{lem:bounded}, hence it will end up in $\itvl{1} \cup \itvl{2} \cup \itvl{3}$. 

  \noindent \underline{\textbf{Proof of P2}:} For any $\theta_{t} \in (1,1.162)$, $1 < \gamma < 1.25$, since $\hgamma(.)$ is decreasing in (1,1.162), we have  $\hgamma(1.162) < \hgamma(\theta_t) < \hgamma(1)$. Also $\hgamma(\theta)$ is linear in gamma with negative coefficient for  $\theta > 1$. Hence it decreases as $\gamma$ increases. Using this, 
\eqals{ 
    .652 = \hgam{1.25}(1.162) < \hgamma(1.162) < \hgamma(\theta_t) < \hgamma(1) = 1.
}  Hence, $ \theta_{t+1} \in \itvl{1} \cup \itvl{2}  $.

  \noindent \underline{\textbf{Proof of P3}:} The proof of this follows from Lemma~\ref{prop:hoh}. 

  \noindent \underline{\textbf{Proof of P4}:} The proof of this follows from Lemma~\ref{lem:closureI1}. 
\end{proof}

\begin{lemma}
  For any $\theta \in  \itvl{1} \cup \itvl{2} $ and any $a,b \in \Gamma$,  $ \hgam{a}(\hgam{b}(\theta)) \in \itvl{1} \cup \itvl{2} $, 
  \begin{align}\label{eq:proof-master}
    \hgam{\gamma_{\mathrm{max}}}(\hgam{\gamma_{\mathrm{max}}}(\theta)) \leq \hgam{a}(\hgam{b}(\theta)) \leq \hgam{\gamma_{\mathrm{min}}}(\hgam{\gamma_\mathrm{min}}(\theta)).
  \end{align}
\end{lemma}
\begin{proof}
For any $ \gamma \in \Gamma $, recall
\eqals{  \hgamma(\theta) &= \theta + \gamma \theta ( 1 - \theta^2) = 1 + (1-\theta) (\gamma \theta (1+\theta) - 1). } Note that for $\theta \in \itvl{1} \cup \itvl{2} $, $ \theta(1+\theta) >  1$, Hence $\gamma \theta (1+\theta) > 1$. This gives us that $\hgamma(\theta) > 1$.  Now we will track where $\theta \in \itvl{1} \cup \itvl{2} $ can end up after two stochastic gradient steps.  
\begin{itemize}
  \item For any $b \in \Gamma$, as $\theta \in \itvl{1} \cup \itvl{2}$, we have $$ \hgam{\gamma_{\mathrm{max}}}(\theta) \geq \hgam{b}(\theta) \geq \hgam{\gamma_\mathrm{min}}(\theta) > 1, $$ note $\hgam{\gamma_{\mathrm{max}}}(\theta) \geq \hgam{b}(\theta) \geq \hgam{\gamma_\mathrm{min}}(\theta)$  holds since $\theta < 1$. 
  \item Now for any $a \in \Gamma$ and $x > 1$, $\hgam{a}(x)$ is a decreasing function in $x$. Hence $$ \hgam{a}(\hgam{\gamma_{\mathrm{max}}}(\theta)) \leq \hgam{a}(\hgam{b}(\theta)) \leq \hgam{a}(\hgam{\gamma_\mathrm{min}}(\theta)). $$ Using $\gamma_{\mathrm{min}} \leq a $, $ \hgam{a}(\hgam{\gamma_\mathrm{min}}(\theta)) \leq \hgam{\gamma_\mathrm{min}}(\hgam{\gamma_\mathrm{min}}(\theta))$, Similarly using $\gamma_{\mathrm{max}} > a$, we have,  $\hgam{\gamma_{\mathrm{max}}}(\hgam{\gamma_{\mathrm{max}}}(\theta)) \leq \hgam{a}(\hgam{\gamma_{\mathrm{max}}}(\theta)) $. Combining them we get, 
  \begin{align}
    \hgam{\gamma_{\mathrm{max}}}(\hgam{\gamma_{\mathrm{max}}}(\theta)) \leq \hgam{a}(\hgam{b}(\theta)) \leq \hgam{\gamma_{\mathrm{min}}}(\hgam{\gamma_\mathrm{min}}(\theta)).
  \end{align}
\end{itemize}
Similar argument can extend it to, 
  \begin{align} \label{eq:lower-upper}
    \hgam{1.25}(\hgam{1.25}(\theta)) < \hgam{a}(\hgam{b}(\theta)) < \hgam{1}(\hgam{1}(\theta)).
  \end{align}
\end{proof}

\begin{lemma}\label{prop:hoh} 
  Let $\theta_{t} \in \itvl{2}$, there exists $k>0$ such that $\theta_{t+2k} \in \itvl{1}$. 
\end{lemma}
\begin{proof} For any $\gamma \in \Gamma$, let $\thetap = \hgamma(\theta) $, then we have
\eqals{ \hgamma(\hgamma(\theta)) - \theta = \hgamma(\theta_{+}) - \theta= \gamma \theta ( 1 - \theta^2) + \gamma  \thetap ( 1 - \thetap^2). 
}Furthermore,
\eqals{    \thetap &=  \theta + \gamma \theta ( 1 - \theta^2) = \theta( 1 + \gamma (1-\theta^2) ), \\
    1 + \theta_{+} &= 1 + \theta + \gamma \theta ( 1 - \theta^2) =  (1+\theta) ( 1 + \gamma \theta (1 - \theta) ), \\
    1 - \theta_{+} &= 1 - \theta - \gamma \theta ( 1 - \theta^2) =  (1-\theta) ( 1 - \gamma \theta (1 + \theta) ). 
}    And multiplying the above three terms and adding  $\theta (1-\theta^2)$, we get, 
\eqals{    \thetap ( 1- \thetap^2) + \theta (1-\theta^2)  &= \theta (1-\theta^2) \lbrace 1 +  \underbrace{\left[ ( 1 + \gamma (1-\theta^2) ) ( 1 + \gamma \theta (1 - \theta)) ( 1 - \gamma \theta (1 + \theta) ) \right]}_{P(\theta)} \rbrace
}
For $\theta \in \itvl{2}$, using $\gamma_{\mathrm{min}}(2-\epsilon)(1-\epsilon) > 2$, we have the inequalities
\eqals{  ( 1 + \gamma (1-\theta^2) ) ( 1 + \gamma \theta (1 - \theta)) > 1, \\
  ( 1 - \gamma \theta (1 + \theta) ) < 1 - \gamma_{min} (2 - \epsilon) (1-\epsilon) < -1 ,  \\
  P(\theta) < -1. 
}
Hence, 
\begin{align}
\label{eq:i2}
  \hgamma(\hgamma(\theta)) - \theta =  \gamma (1-\theta^2) ( 1 + P(\theta)) < 0.
\end{align}
Therefore, for $[1-\epsilon,1)$, for any $\gamma \in \Gamma$, $ \hgamma(\hgamma(\theta)) < \theta $. Hence for any two stochastic gradient step with $a,b \in \Gamma$, from \eqref{eq:proof-master}, $ \theta_{t+2} = \hgam{a}(\hgam{b}(\theta_t)) \leq \hgam{\gamma_{\mathrm{min}}}(\hgam{\gamma_{\mathrm{min}}}(\theta_t)) < \theta_{t} $.  From any point in $\itvl{2}$, we have $| \theta_{t+2} - 1| > |\theta_{t} - 1|$, for any $a,b \in \Gamma$. Intutively this means that in \textit{two gradient steps} the iterates move \textit{further away from 1} until it eventually leaves the interval $\itvl{2}$ as the sequence $\lbrace\theta_{t+2k}\rbrace_{k \geq 0}$ is strictly decreasing with no limit point in $\itvl{2}$. From Lemma~\ref{lem:closureI} , we know that in two steps the iterates will never leave $\itvl{1} \cup \itvl{2}$. Hence they will eventually end up in $\itvl{1}$ leaving $\itvl{2}$. 
\end{proof}

\begin{property} Define $ \ggamma(\theta) := \hgamma(\hgamma(\theta)) $ for the sake of brevity. The followings properties hold for $\theta \in \itvl{1} \cup \itvl{2}$, $\gamma \in \Gamma$ and  $\theta_{\gamma}$ the root of $~\hgamma^{'}(\theta)$:
 \begin{description}
    \item[Q1]  $\ggamma(\theta) \geq \ggamma(\theta_{\gamma}) $.
    \item[Q2] The function $\ggamma(.)$ is decreasing in $[0.65,\theta_{\gamma})$ and increasing in $(\theta_{\gamma},1]$. 
  \end{description}
\end{property}

\begin{proof}
  Note $\hgamma^{'}(\theta) =  1 + \gamma - \gamma 3 \theta^2 $ has at most one root $\theta_{\gamma} \in (0,1)$. Note that for all $\gamma \in \Gamma$, $\theta_{\gamma} \in \itvl{1} \cup \itvl{2}$. For any $\gamma$, $\ggamma^{'}(\theta) =\hgamma^{'}(\hgamma(\theta)) \hgamma^{'}(\theta)$. For any $\theta \in \itvl{1} \cup \itvl{2}$, we have, $\hgamma(\theta) > 1 \implies \hgamma^{'}(\hgamma(\theta)) < 0$. Therefore, $\ggamma^{'}(\theta)$ has only one root in $\itvl{1} \cup \itvl{2}$. Since  $\theta_{\gamma} \in \itvl{1} \cup \itvl{2}$, note $\ggamma^{''}(\theta_{\gamma}) = \hgamma^{'}(\hgamma(\theta_{\gamma})) \hgamma^{''}(\theta_{\gamma}) > 0 $. Therefore,  $\ggamma(.)$ attains its minimum at $\theta_{\gamma}$ and this shows the desired properties.
  \end{proof}

  \begin{lemma} \label{lem:closureI}
    For any $\theta \in \itvl{1} \cup \itvl{2} $ and any $a,b \in \Gamma$,  $ \hgam{a}(\hgam{b}(\theta)) \in \itvl{1} \cup \itvl{2}$. 
  \end{lemma}
  \begin{proof}
  
  \noindent\underline{\textbf{Lower Bound:}}   From \eqref{eq:lower-upper}, we know
\eqals{    \hgam{1.25}(\hgam{1.25}(\theta)) < \hgam{a}(\hgam{b}(\theta)) }
  We know that from property \textbf{Q1} that $ \ggamma(\theta) \geq \ggamma(\theta_{\gamma})$. 
Hence 
\eqals{  \ggam{1.25}(\theta_{1.25}) < \ggam{1.25}(\theta) < \hgam{a}(\hgam{b}(\theta)) }
It can be quickly checked that $ .65 < \ggam{1.25}(\theta_{1.25}) $. Hence the lower bound holds. 

\noindent\underline{\textbf{Upper Bound:}}  From \eqref{eq:lower-upper}, we know
\eqals{   \hgam{a}(\hgam{b}(\theta)) < \hgam{1}(\hgam{1}(\theta)) }
We know that from property \textbf{Q2} that $ \ggam{1}(\theta) \leq \textrm{max} \{ \ggam{1}(1),  \ggam{1}(0.65) \}$. It can be easily verified that $ \ggam{1}(0.65) < 0.98 $. Hence $ \ggam{1}(\theta) < 1$.  

\end{proof}

\begin{lemma}  \label{lem:closureI1}    For any $\theta \in \invl_{1}$ and any $a,b \in \Gamma$,  $ \hgam{a}(\hgam{b}(\theta)) \in \invl_{1}$ and $\hgam{a}(\theta) \in (1+\epsilon,1.162) $. 
\end{lemma}
\begin{proof}
  The lower bound in Lemma~\ref{lem:closureI} holds here. For the upper  bound, from  and \eqref{eq:proof-master}, 
\eqals{    \hgam{a}(\hgam{b}(\theta))  \leq \hgam{\gamma_{\mathrm{min}}}(\hgam{\gamma_\mathrm{min}}(\theta)).}  Using property \textbf{Q2}, 
\eqals{    \hgam{\gamma_{\mathrm{min}}}(\hgam{\gamma_\mathrm{min}}(\theta)) \leq  \textrm{max} \{ \ggam{\gamma_\mathrm{min}}(1 - \epsilon),  \ggam{\gamma_\mathrm{min}}(0.65) \}
}  From \eqref{eq:i2}, $\ggam{\gamma_\mathrm{min}}(1 - \epsilon) < 1- \epsilon $. From \eqref{eq:lower-upper}, $\ggam{\gamma_\mathrm{min}}(0.65) < \ggam{1}(0.65) < 0.98 < 1 - \epsilon $. 
  In $\itvl{1}$, the function $\hgam{a}(.)$ first increases reaches maximum and decreases. Hence for $\theta \in \itvl{1}$, $ \hgam{a}(\theta) \geq \mathrm{min}\{ \hgam{a}(0.65), \hgam{a}(1-\epsilon) \}$ .
\eqals{    \hgam{a}(1-\epsilon) &\geq 1 - \epsilon + a (1-(1-\epsilon)^2)(1-\epsilon), \\
    &=  1 - \epsilon + a ( 2\epsilon - \epsilon^2)(1-\epsilon), \\
    &\geq 1 - \epsilon + \gamma_{\mathrm{min}} ( 2\epsilon - \epsilon^2)(1-\epsilon), \\
    &= 1 + \epsilon + \epsilon \left( \gamma_{\mathrm{min}} ( 2 - \epsilon)(1-\epsilon) -2 \right) > 1 + \epsilon.
}
Also $\hgam{a}(0.65) > \hgam{1}(0.65) > 1.02> 1+ \epsilon$, therefore $\hgam{a}(\theta) > 1+\epsilon$ and this completes the proof. 
\end{proof}

\section{\upd{Empirical Validation of the SDE Modeling}}
\label{app:sec:sde_validation}
\upd{
In this section, we experimentally check the validity of the SDE modeling of SGD in \eqref{eq:effective_dynamics_general} in terms of the key metrics: training loss,  test loss, rank of the Jacobian, and feature sparsity.
}

\upd{
\myparagraph{SDE discretization.}
Let $\gamma_t$ be the SDE discretization step size, $\eta_t$ the step size of the corresponding SGD that we aim to validate, $\delta_t$ the noise intensity level, and $Z_t \sim \mathcal{N}(0, I_n)$. Then we discretize the SDE from \eqref{eq:effective_dynamics_general} as follows:
\begin{align}
    \label{eq:sde_discretization}
    \theta_{t+1} = \theta_t - \gamma_t \nabla_\theta \LL(\theta_t) + \sqrt{\gamma_t} \sqrt{\eta_t \delta_t}\, \phi_{\theta_t} (X)^\top Z_t.
\end{align}
To approximate continuous time, we use a small discretization step size $\gamma_t := \eta_t / 10$ and run the discretization for $10\times$ longer than the corresponding SGD run.
We use $\eta_t := \eta_{\floor{t/10}}^{SGD}$ and $\delta_t := c \cdot \LL(\theta_{\floor{t/10}}^{SGD})$ where $c$ is a constant that we select for each setting separately to match the training dynamics of the corresponding SGD run. In addition, we also evaluate a discretization of gradient flow (i.e., \eqref{eq:sde_discretization} without the noise term) which helps to draw conclusions about the role of the noise term.
}

\begin{figure}[t!]
    \centering
    \footnotesize
    \upd{\textbf{Diagonal linear networks}}\\
    \includegraphics[width=0.99\textwidth]{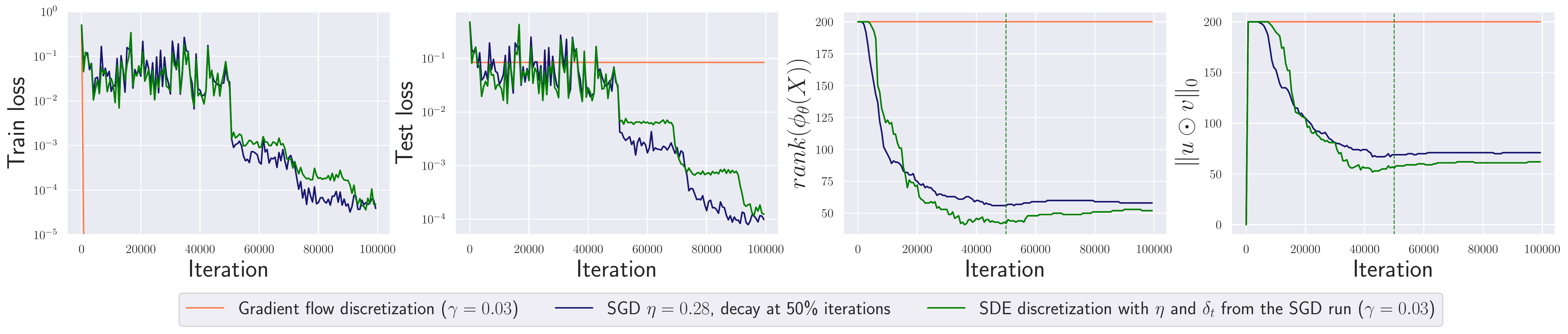}\\
    \vspace{2mm}
    \footnotesize
    \upd{\textbf{Two-layer ReLU networks on 1D regression}}\\
    \includegraphics[width=0.99\textwidth]{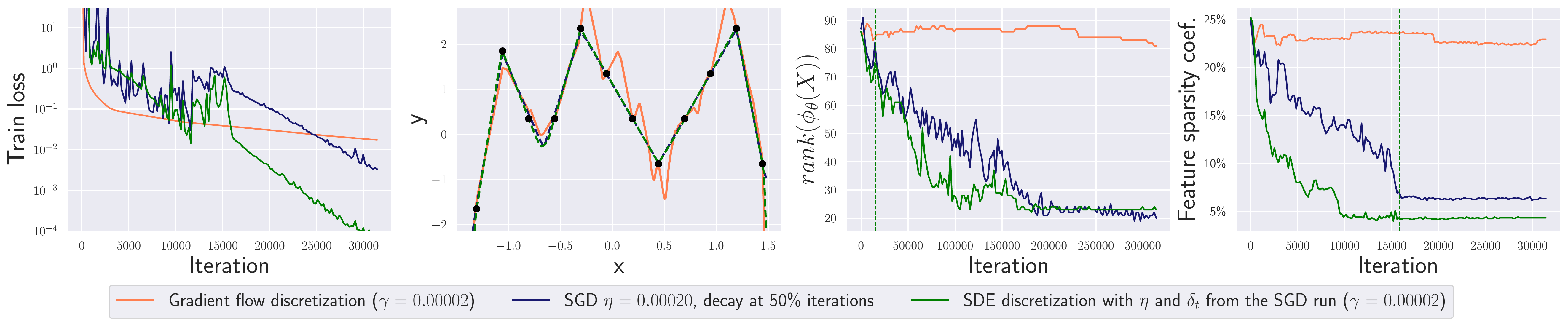}\\
    \vspace{2mm}
    \footnotesize
    \upd{\textbf{Two-layer ReLU networks in a teacher-student setup}}\\
    \includegraphics[width=0.99\textwidth]{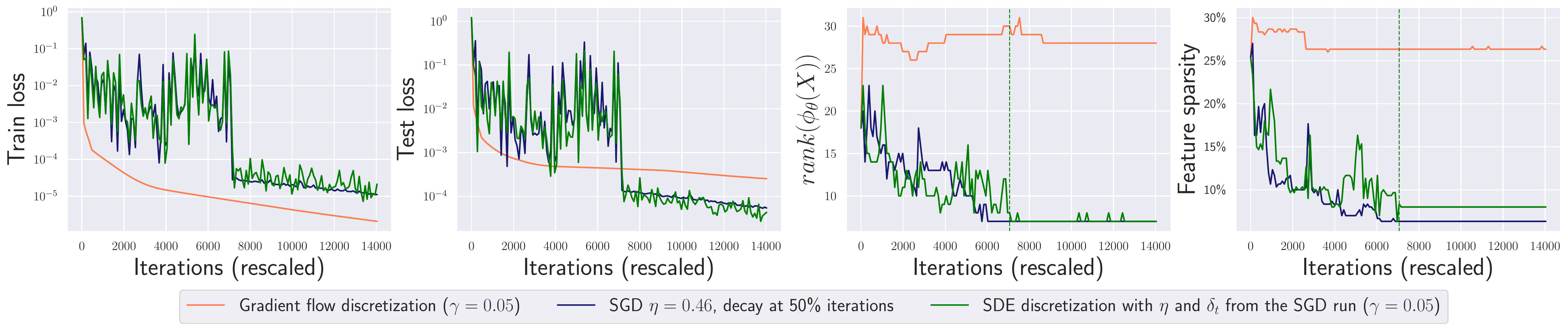}\\
    \vspace{2mm}
    \footnotesize
    \upd{\textbf{Three-layer ReLU networks in a teacher-student setup}}\\
    \includegraphics[width=0.99\textwidth]{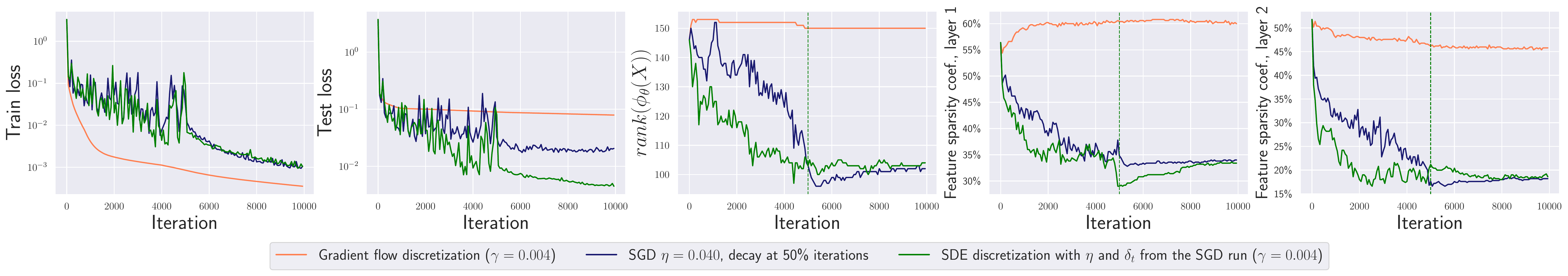}
    \caption{\upd{\textbf{Empirical validation of the SDE modeling.} In all cases, the dynamics of the SDE discretization qualitatively matches the dynamics of the corresponding SGD run. Moreover, gradient flow discretization exhibits no rank minimization or feature sparsity which suggests that the presence of the noise plays a key role in learning sparse features.}}
    \label{fig:sde_validation}
\end{figure}
\myparagraph{Experimental results.}
We present the discretization results in Fig.~\ref{fig:sde_validation} for all models considered in the paper except deep networks for which computing the Jacobian $\phi_{\theta_t}$ on each iteration of the SDE discretization is too costly. In all cases, the dynamics of the SDE discretization qualitatively matches the dynamics of the corresponding SGD run. In particular, we observe similar levels of decrease in the rank of the Jacobian and feature sparsity coefficient. We note that the match between SDE and SGD curves is not expected to be precise due to the inherent randomness of the process.
Finally, we observe that gradient flow discretization exhibits no rank minimization or feature sparsity which suggests that the presence of the noise (either from the original SGD or its SDE discretization) plays a key role in learning sparse features.

\section{Additional Experimental Results}
\label{app:sec:additional_experiments}
This section of the appendix presents additional experiments complementing the ones presented in the main text.

\begin{wrapfigure}{r}{0.51\textwidth}
    \vspace{-2mm}
    \centering
    \includegraphics[width=0.25\textwidth]{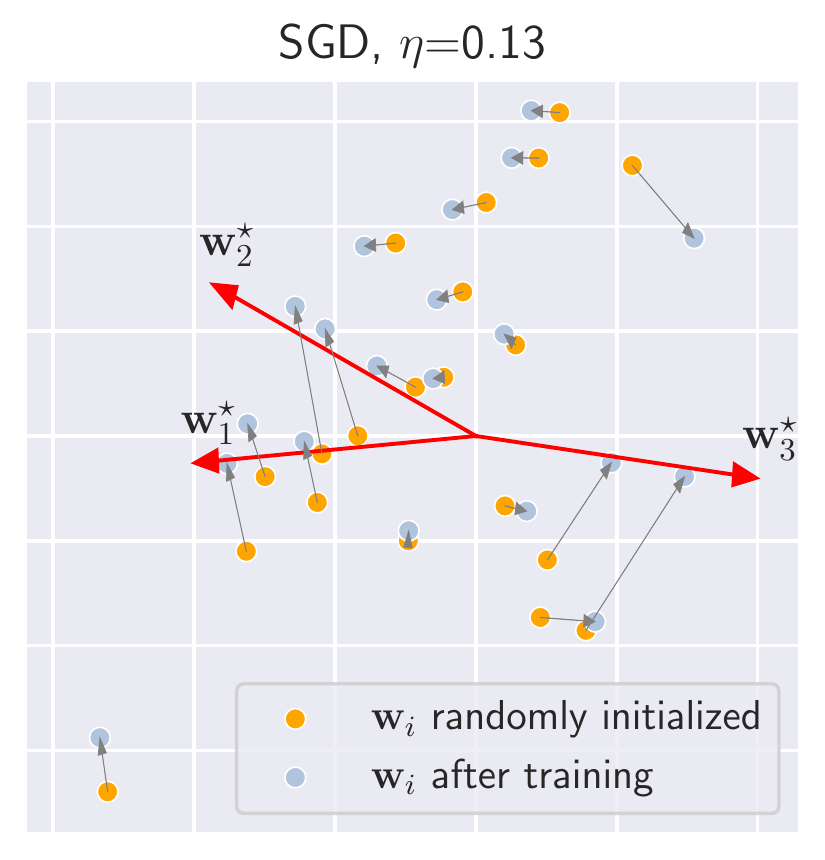}
    \includegraphics[width=0.25\textwidth]{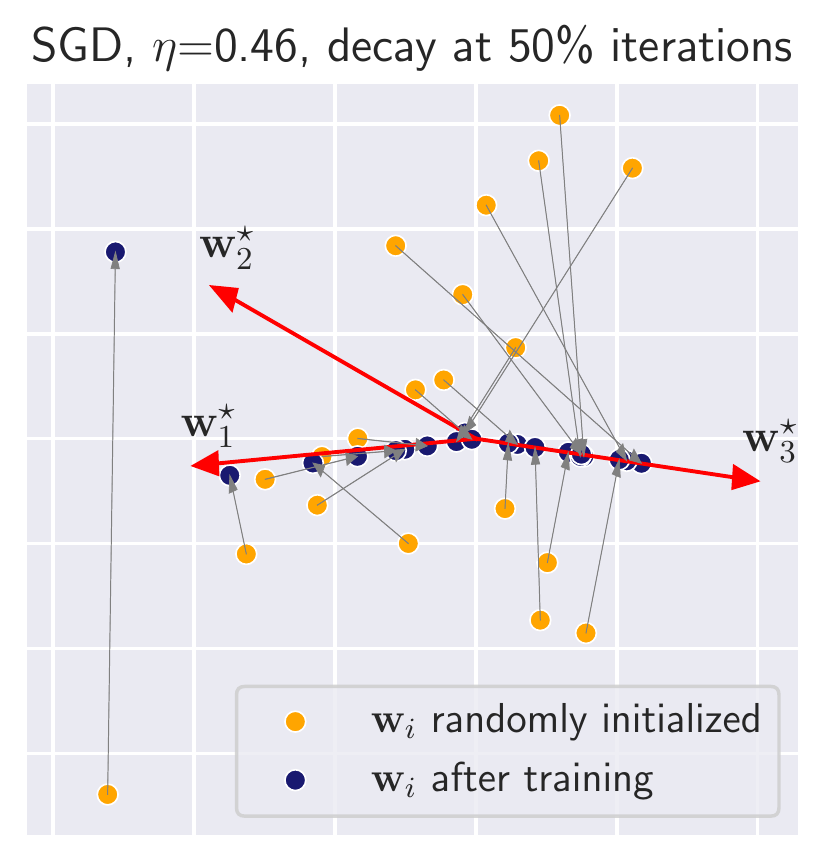}
    \vspace{-4mm}
    \caption{
    Only for a large step size, the neurons $w_i$ cluster along the teacher neurons $w_i^\star$ leading to a model that uses a sparse set of features.
    }
    \label{fig:2d_neuron_movement}
    \vspace{-3mm}
\end{wrapfigure}
\myparagraph{Illustration of neuron dynamics.} 
We illustrate the change of neurons during training of two-layer ReLU networks (without biases) in the teacher-student setup of \citet{chizat2019lazy} (see Fig.~1 therein) using a large initialization scale for which small step sizes of GD or SGD lead to lazy training. %
We illustrate \textbf{(O1)}--\textbf{(O3)} in Fig.~\ref{fig:fc_net_teacher_student_d2} and show neuron dynamics in Fig.~\ref{fig:2d_neuron_movement}. 
We see that for SGD with a small step size, the neurons $w_i$ stay close to their initialization, while for a large step size, there is a clear clustering of directions $w_i$ along the teacher directions $w_i^\star$.
The overall picture is very similar to Fig.~1 of \citet{chizat2019lazy} where the same feature learning effect is achieved via gradient flow from a small initialization which is, however, much more computationally expensive due to the saddle point at zero. Finally, we note that the clustering phenomenon of neurons $w_i$ motivates the removal of highly correlated activations in the feature sparsity coefficient: although the corresponding activations are often non-zero, many of them in fact implement \textit{the same feature} and thus should be counted only once.

\myparagraph{Further results.}
We give a short overview of additional figures referred to in the main text. More details can be found in the captions.
\begin{itemize}
    \item Figure~\ref{fig:dln_motivating_picture_gd} shows that even if loss stabilization occurs in diagonal linear networks, the implicit bias towards sparsity is largely weaker than that of SGD and generalization is poor.
    \item Figures~\ref{fig:fc_net_1d_regression_overreg} and~\ref{fig:fc_net_1d_regression_classifier_vis} demonstrate that the implicit bias resulting from high-loss stabilization makes the neural nets learn \textit{first} a simple model \textit{then eventually} fits the data.
    \item Figure~\ref{fig:fc_net_teacher_student_d2} presents the sparsifying effect corresponding to the neurons' movements exhibited in Figure~\ref{fig:2d_neuron_movement}.
    \item Figures~\ref{fig:c10_deep_nets} and \ref{fig:c100_deep_nets} 
    exhibit the feature sparsity in ResNet-18 / ResNet-34 architectures on CIFAR-10 and CIFAR-100 in the basic and state-of-the-art settings.
    \item Figure~\ref{fig:feature_learning_conv_filters_cifar10} showcases the features learning induced by large step sizes for different layers of ResNets-18 when trained on CIFAR-10.
\end{itemize}

\vspace{1cm}

\begin{figure}[ht]
    \centering
    \includegraphics[width=0.99\textwidth]{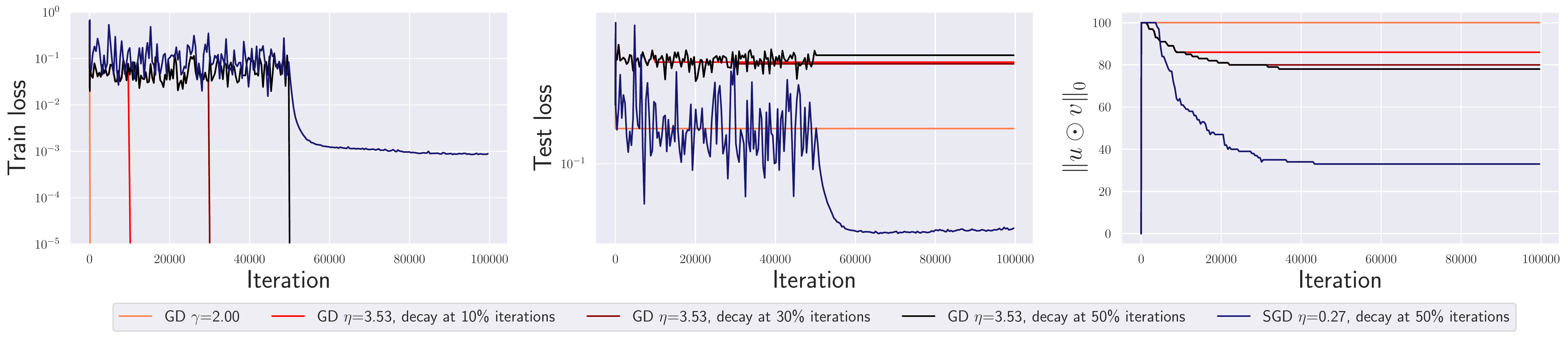}
    \caption{\textbf{Diagonal linear networks}. Loss stabilization also occurs for \textit{full-batch gradient descent} but does not lead to a similar level of sparsity as SGD and also does not improve the test loss.}
    \label{fig:dln_motivating_picture_gd}
\end{figure}

\begin{figure}[ht]
    \centering
    \includegraphics[width=0.99\textwidth]{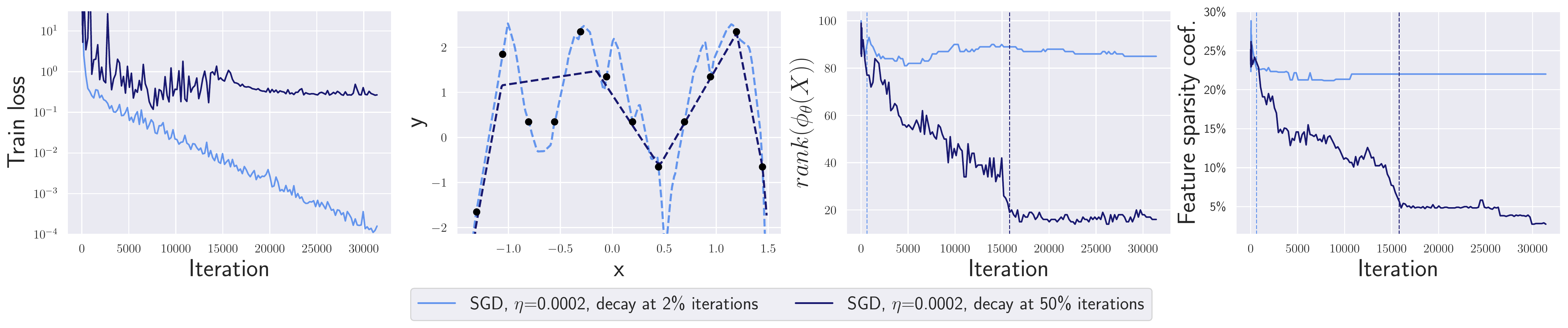}
    \caption{\textbf{Two-layer ReLU networks for 1D regression}. Unlike for Fig.~\ref{fig:fc_net_1d_regression_main_plots}, here we use a larger warmup coefficient ($500\times$ vs. $400\times$) which leads to overregularization such that the 50\%-schedule run fails to fit all the training points and gets stuck at a too high value of the training loss ($\approx 10^{-0.5}$).}
    \label{fig:fc_net_1d_regression_overreg}
\end{figure}

\begin{figure}[ht]
    \centering
    \includegraphics[width=0.99\textwidth]{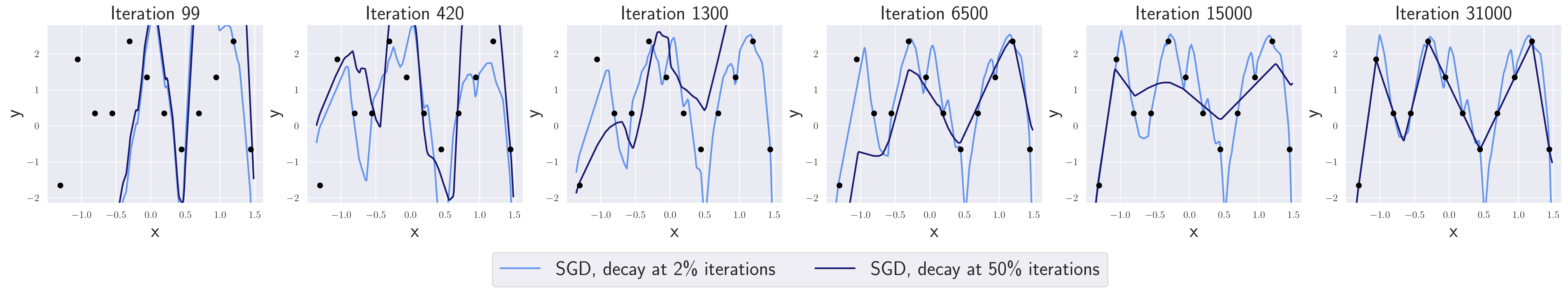}
    \caption{\textbf{Two-layer ReLU networks for 1D regression}. Illustration of the resulting models from Fig.~\ref{fig:fc_net_1d_regression_main_plots} over training iterations. We can see that first the model is simplified and only then it fits the training data.}
    \label{fig:fc_net_1d_regression_classifier_vis}
\end{figure}

\begin{figure}[ht]
    \centering
    \includegraphics[width=0.99\textwidth]{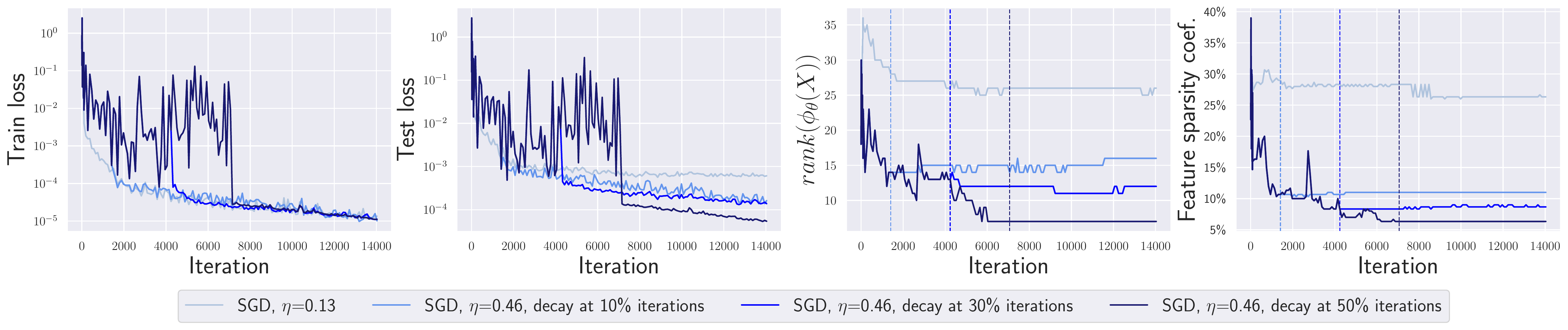}
    \caption{\textbf{Two-layer ReLU networks in a teacher-student setup}. Loss stabilization for two-layer ReLU nets in the teacher-student setup with input dimension $d=2$. We observe loss stabilization, better test loss for longer schedules and sparser features due to simplification of $\phi(X)$.}
    \label{fig:fc_net_teacher_student_d2}
\end{figure}

\begin{figure*}[t]
    \centering
    \footnotesize
    \textbf{ResNet-18 on CIFAR-10, basic setting (no momentum and augmentations)}
    \includegraphics[width=1.0\textwidth]{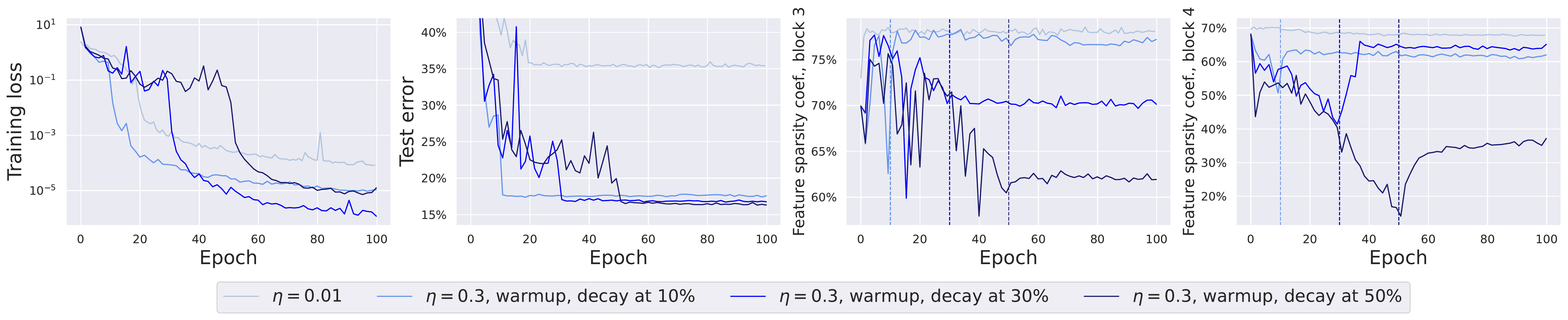}\\
    \textbf{ResNet-18 on CIFAR-10, state-of-the-art setting (with momentum and augmentations)}
    \includegraphics[width=1.0\textwidth]{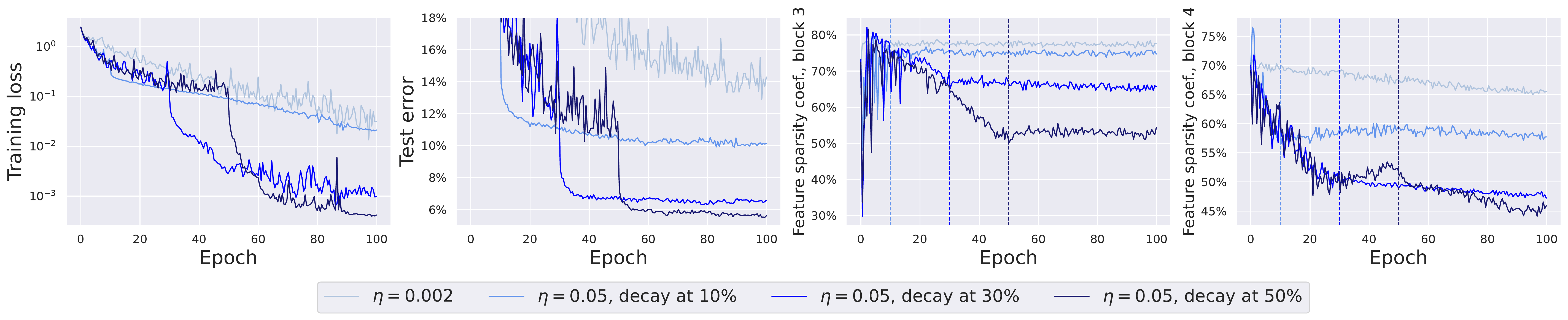}\\
    \vspace{-2mm}
    \caption{\textbf{ResNet-18 trained on CIFAR-10.} Both in the basic and state-of-the-art settings, the training loss stabilizes, the test loss noticeably depends on the length of the schedule, and the feature sparsity \upd{coefficient} is minimized over iterations.}
    \vspace{-2mm}
    \label{fig:c10_deep_nets}
\end{figure*}

\begin{figure}[t]
    \centering
    \textbf{ResNet-34 on CIFAR-100, basic setting (no momentum and augmentations)}
    \includegraphics[width=1.0\textwidth]{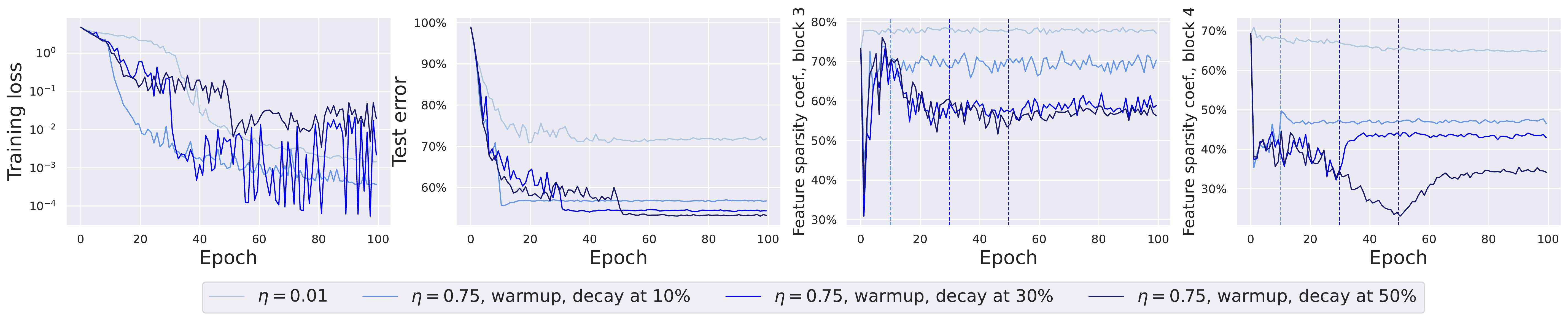}\\
    \vspace{1mm}
    \textbf{ResNet-34 on CIFAR-100, state-of-the-art setting (with momentum and augmentations)}
    \includegraphics[width=1.0\textwidth]{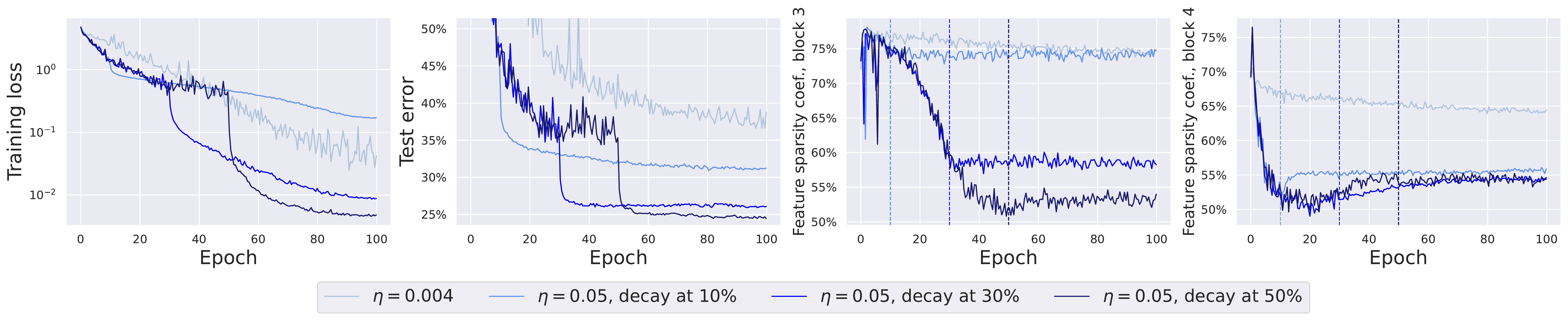}
    \caption{\textbf{ResNet-34 trained on CIFAR-100.} Both in the basic and state-of-the-art settings, the training loss stabilizes, the test loss significantly depends on the length of the schedule, and feature sparsity is minimized over iterations. However, differently from the plots on CIFAR-10, here without explicit regularization we observe oscillating behavior after the step size decay (although at a very low level between $10^{-4}$ and $10^{-2}$).}
    \label{fig:c100_deep_nets}
\end{figure}

\begin{figure}[t]
    \begin{subfigure}[t]{.3\textwidth}
        \centering
        \caption{\textbf{Early layer}}
        \footnotesize Initial \quad \ \ Small $\eta$ \quad \ \ Large $\eta$
        \includegraphics[width=0.32\textwidth]{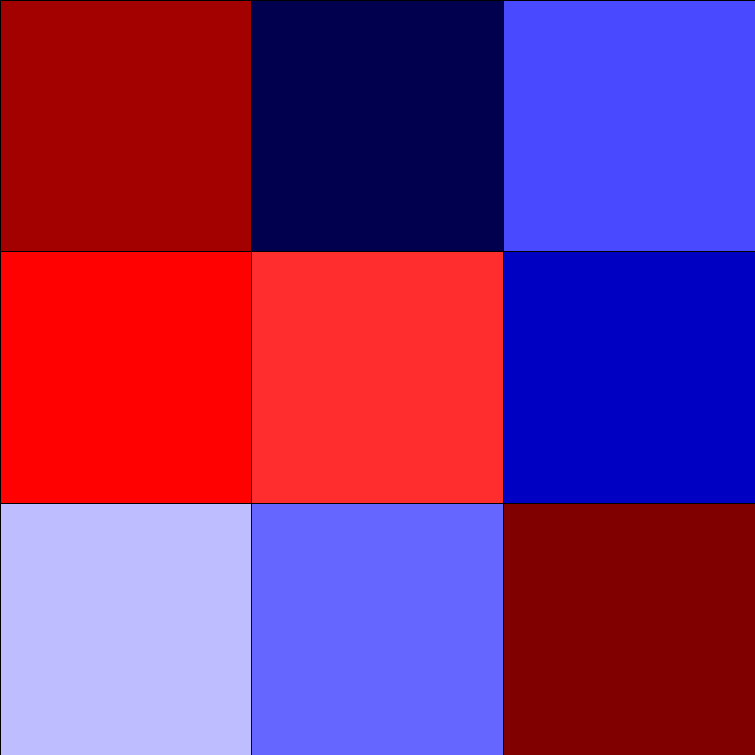}
        \includegraphics[width=0.32\textwidth]{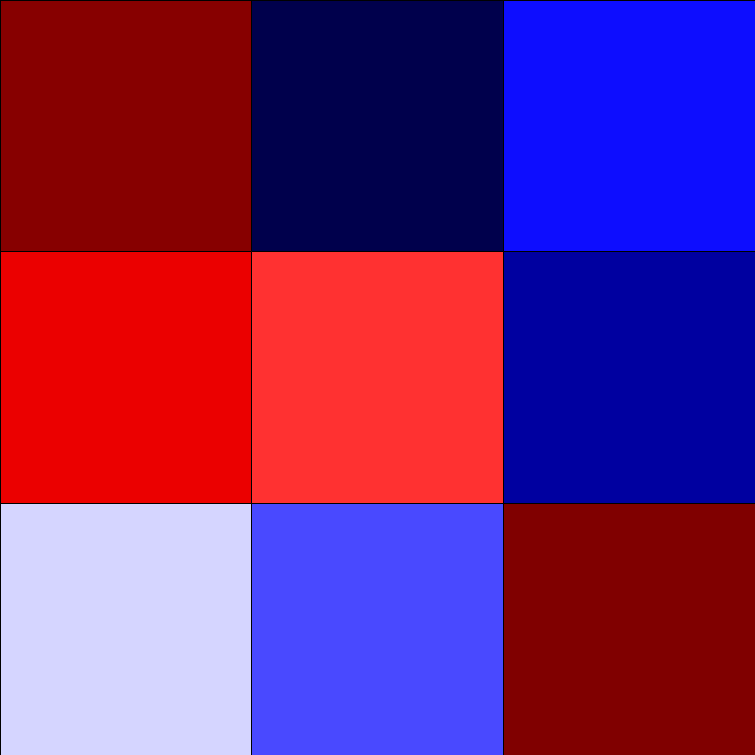}
        \includegraphics[width=0.32\textwidth]{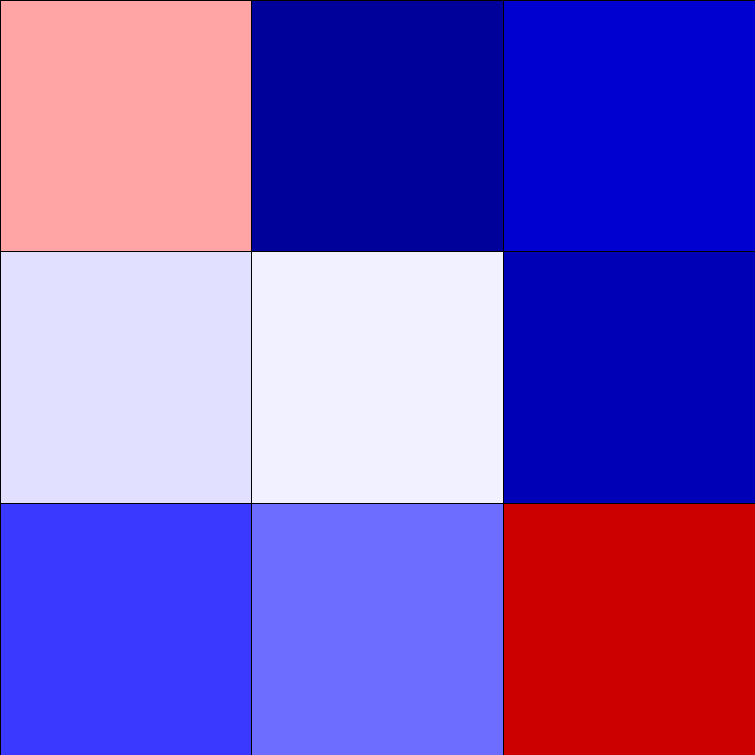}
        \\
        \includegraphics[width=0.32\textwidth]{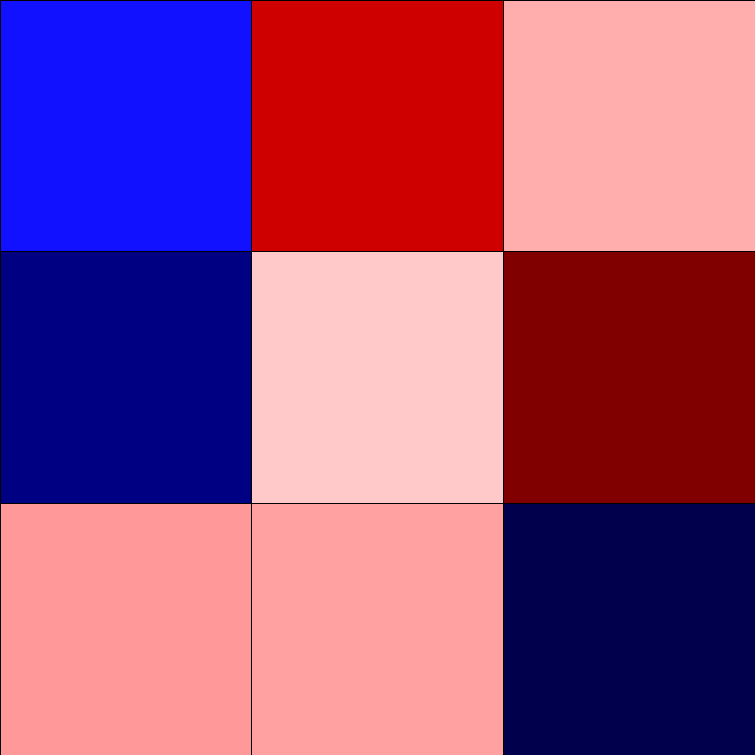}
        \includegraphics[width=0.32\textwidth]{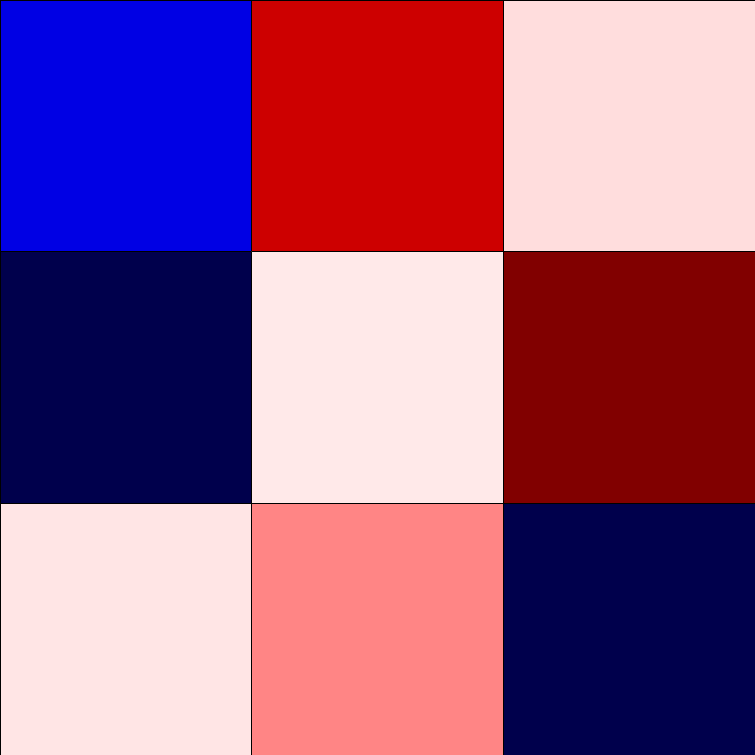}
        \includegraphics[width=0.32\textwidth]{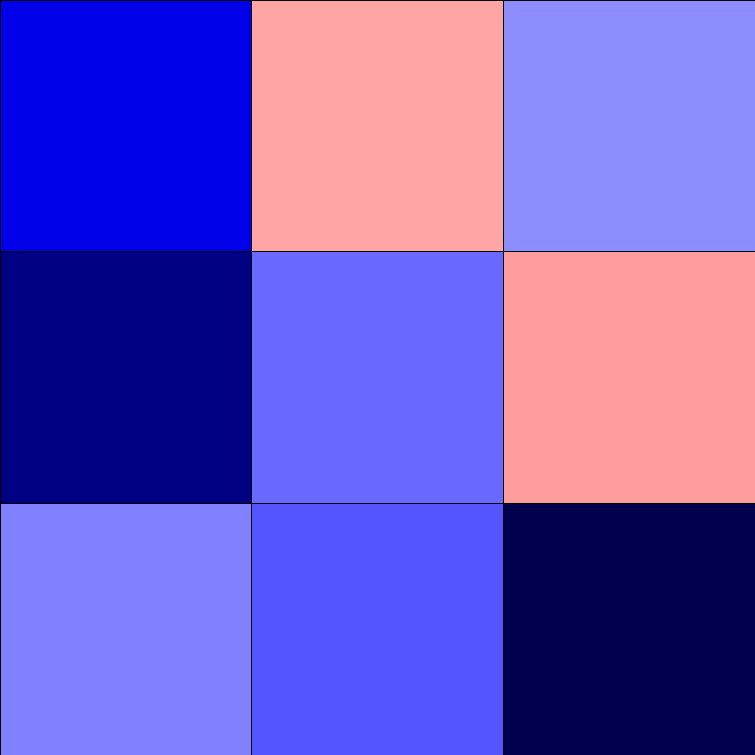}
        \\
        \includegraphics[width=0.32\textwidth]{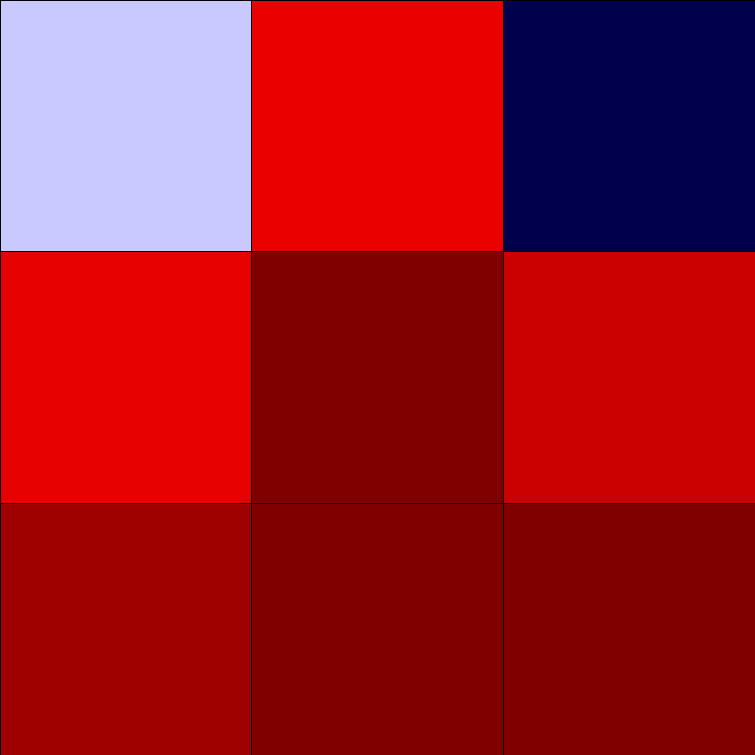}
        \includegraphics[width=0.32\textwidth]{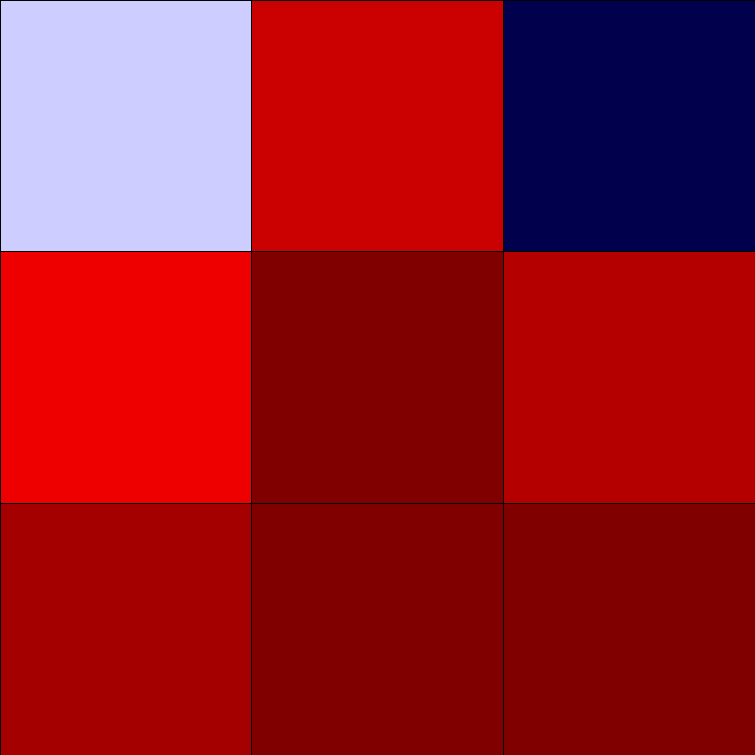}
        \includegraphics[width=0.32\textwidth]{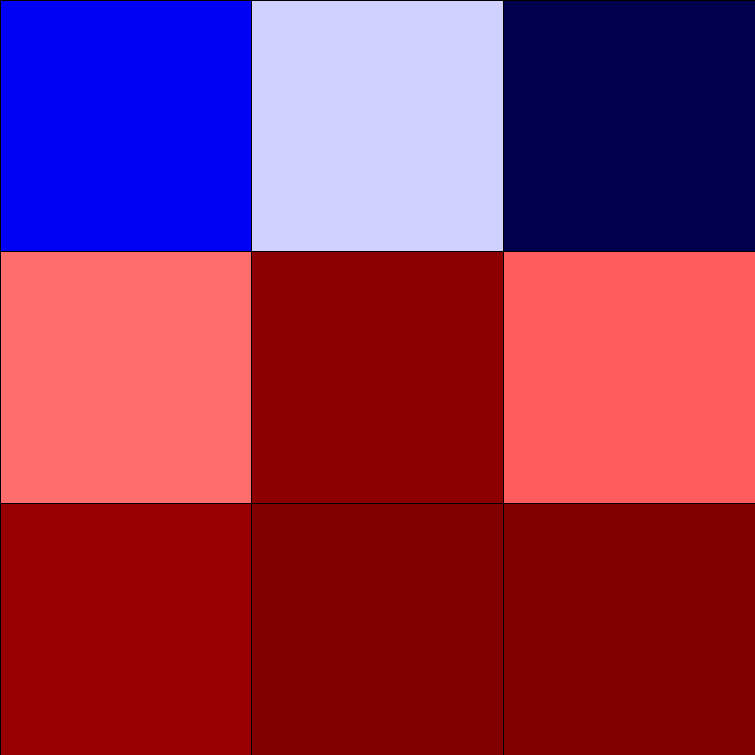}
        \\
        \includegraphics[width=0.32\textwidth]{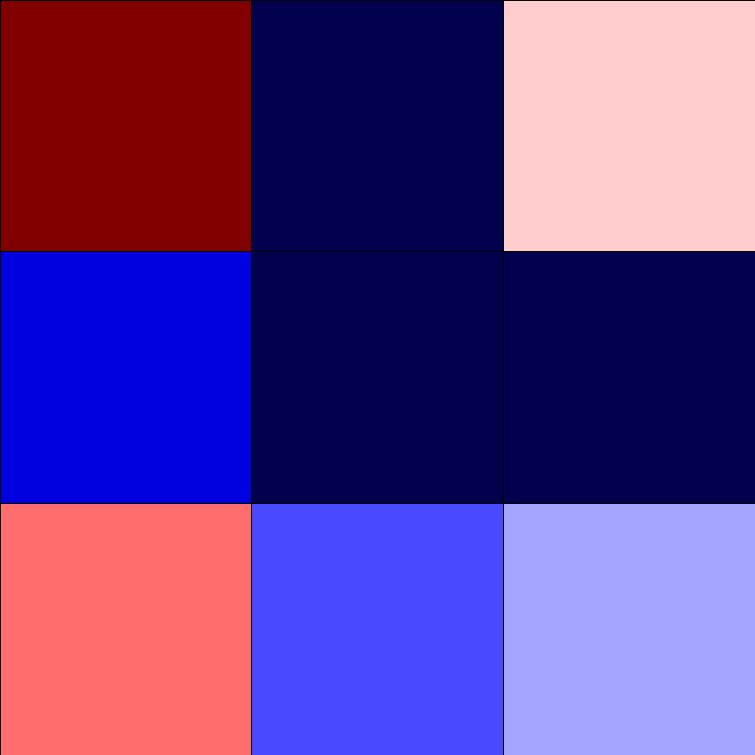}
        \includegraphics[width=0.32\textwidth]{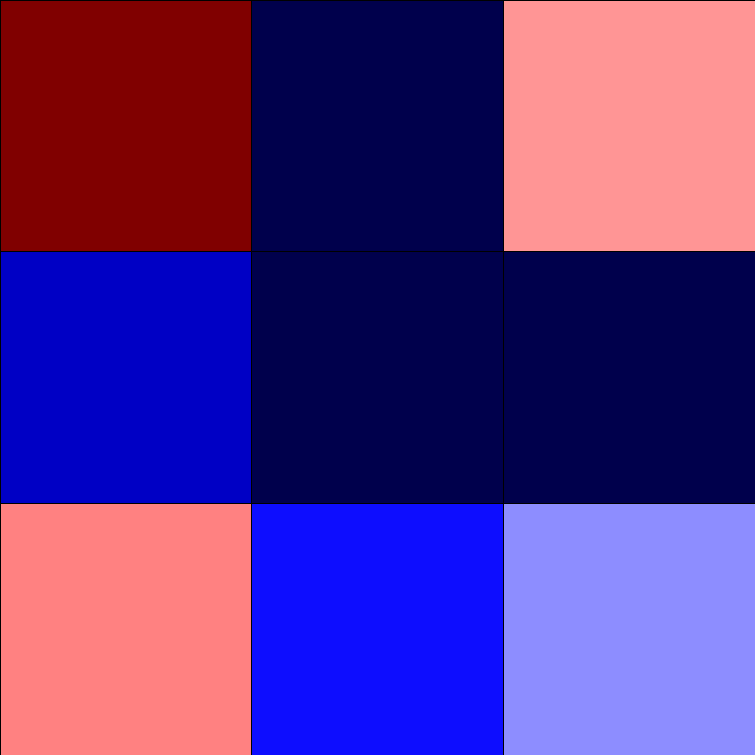}
        \includegraphics[width=0.32\textwidth]{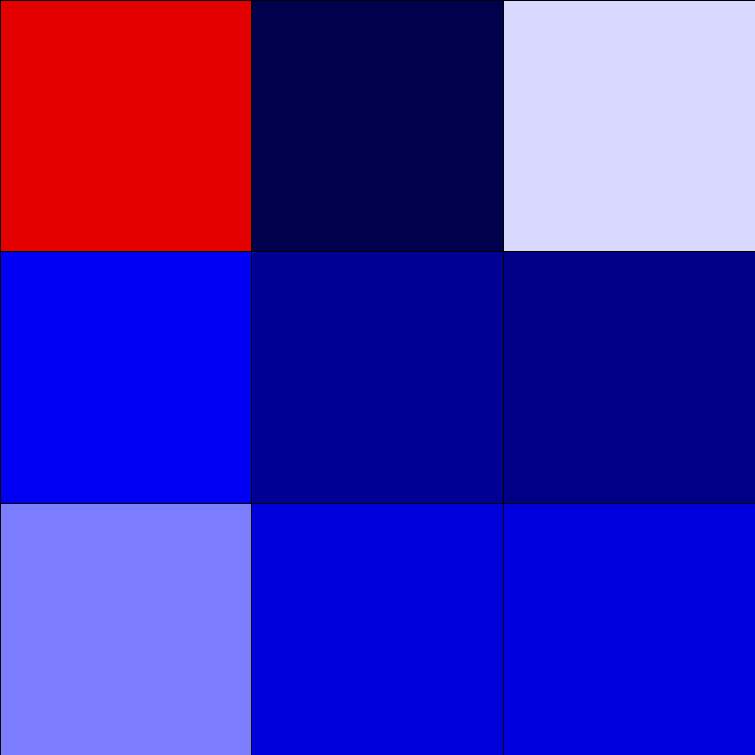}
    \end{subfigure}
    \quad
    \begin{subfigure}[t]{.3\textwidth}
        \centering
        \caption{\textbf{Middle layer}}
        \footnotesize Initial \quad \ \ Small $\eta$ \quad \ \ Large $\eta$
        \includegraphics[width=0.32\textwidth]{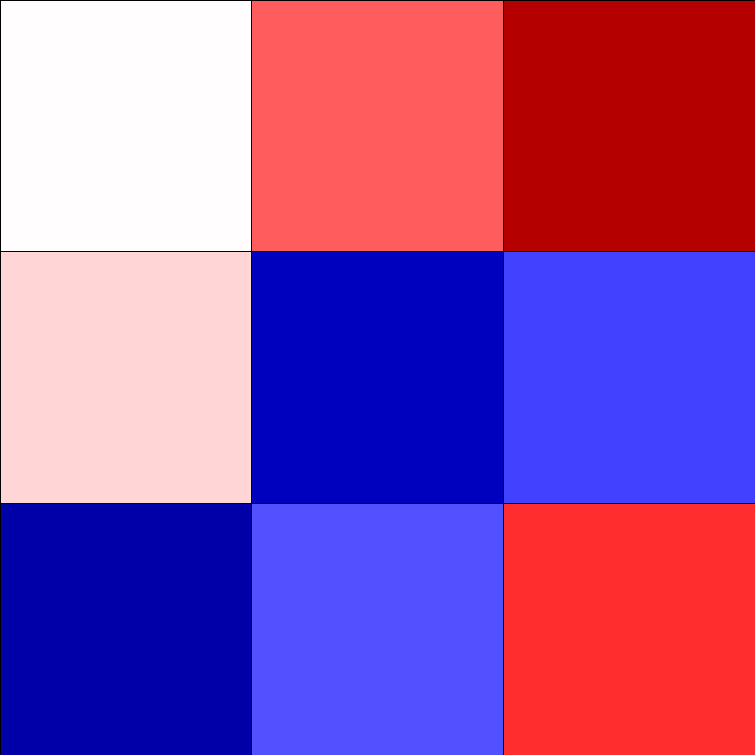}
        \includegraphics[width=0.32\textwidth]{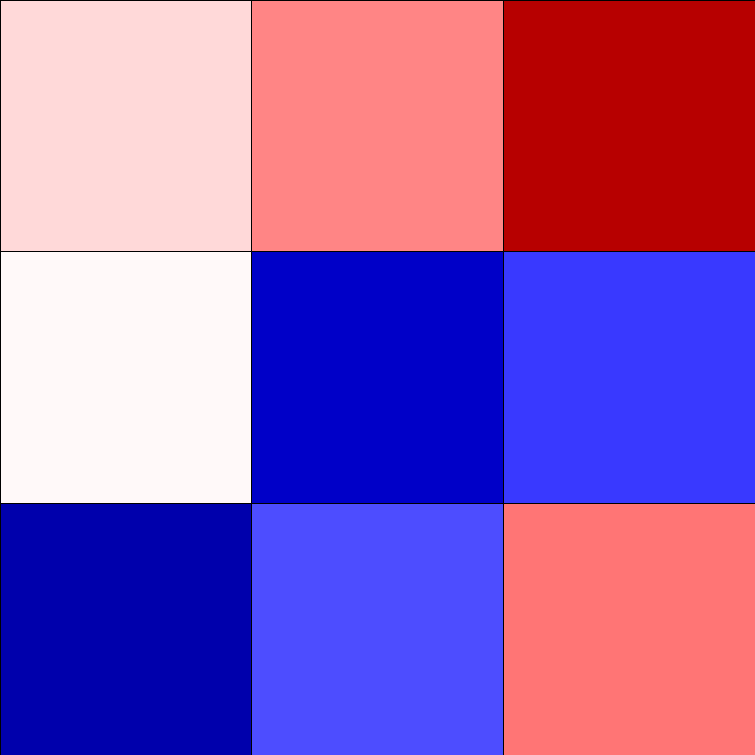}
        \includegraphics[width=0.32\textwidth]{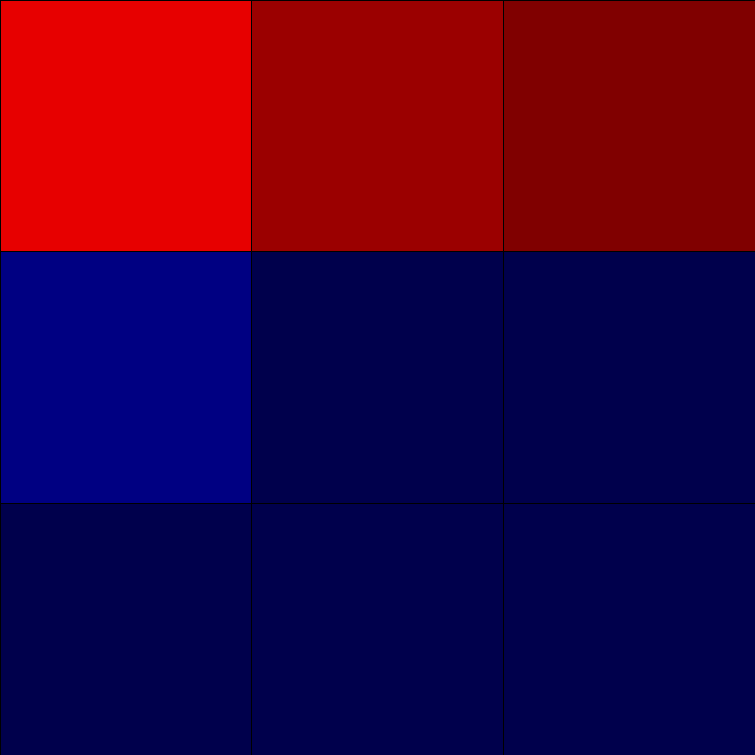}
        \\
        \includegraphics[width=0.32\textwidth]{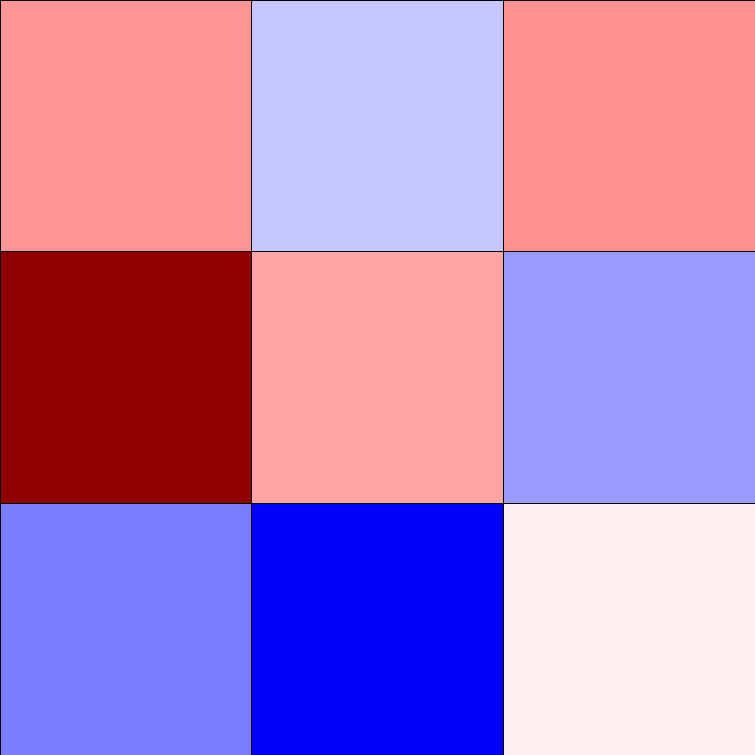}
        \includegraphics[width=0.32\textwidth]{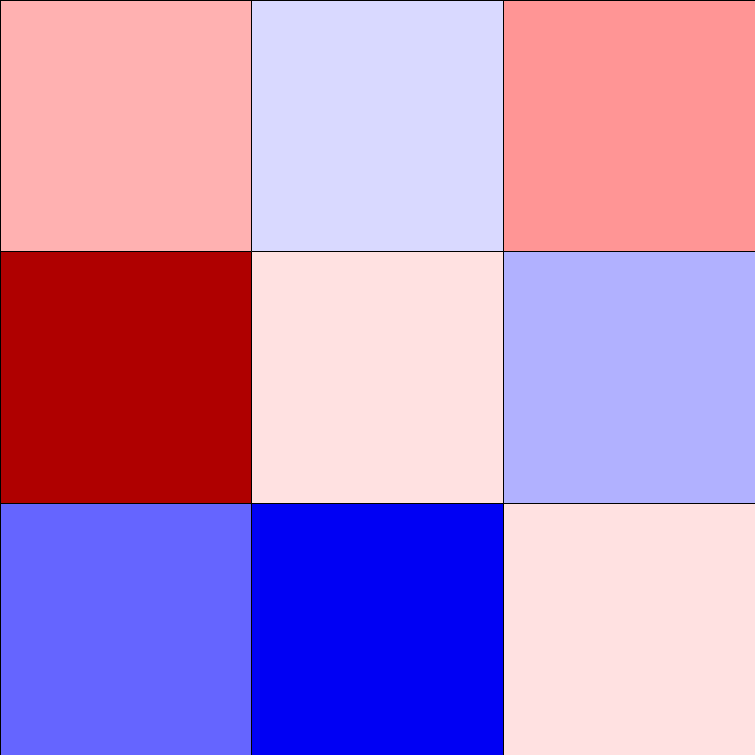}
        \includegraphics[width=0.32\textwidth]{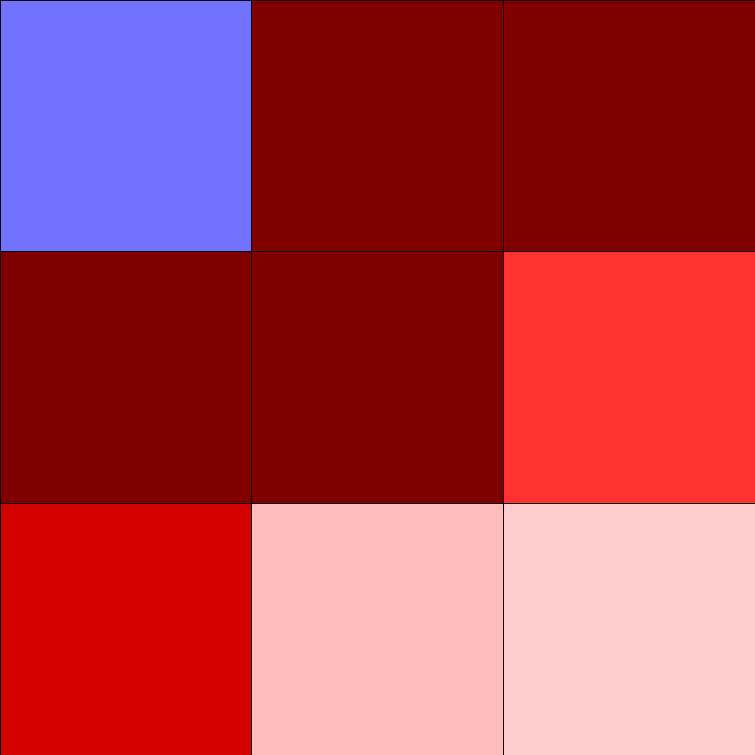}
        \\
        \includegraphics[width=0.32\textwidth]{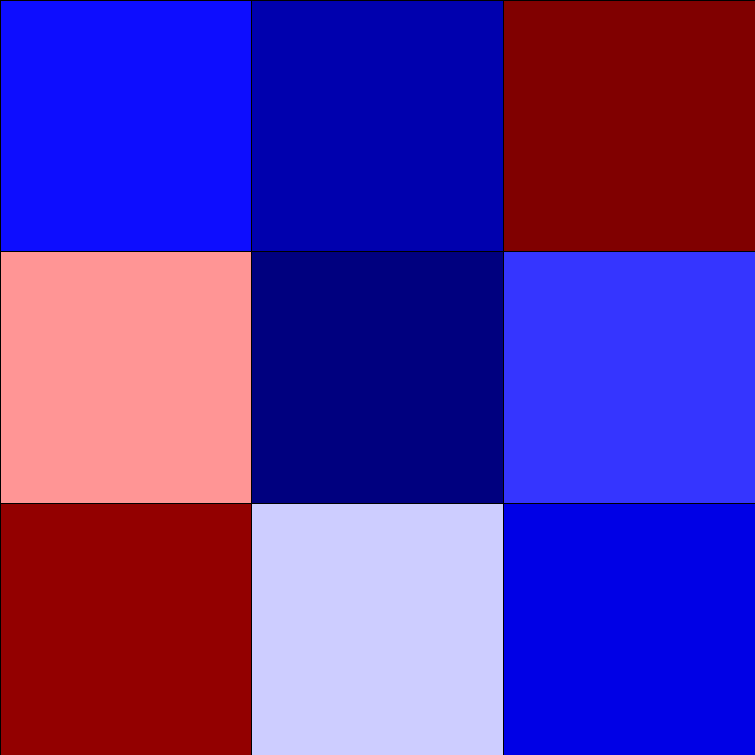}
        \includegraphics[width=0.32\textwidth]{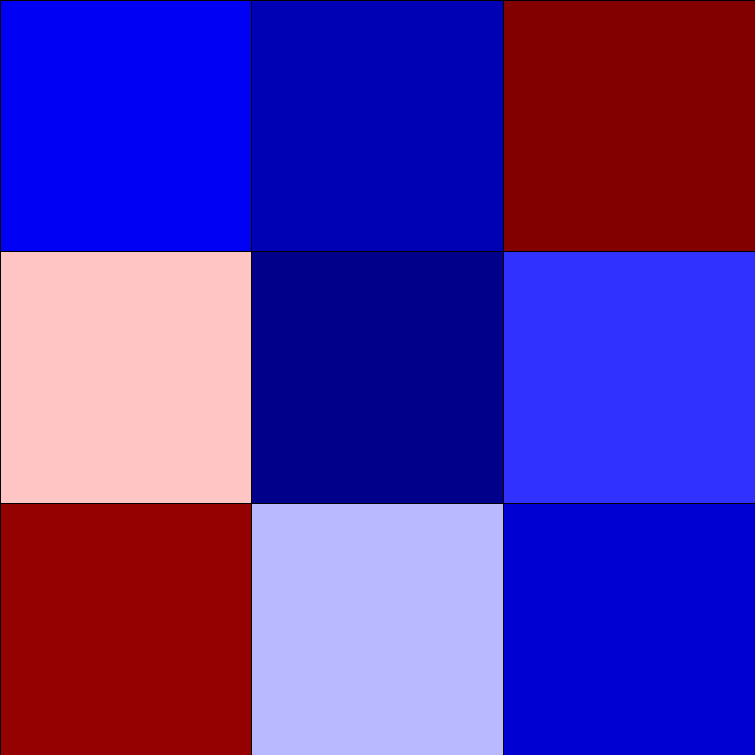}
        \includegraphics[width=0.32\textwidth]{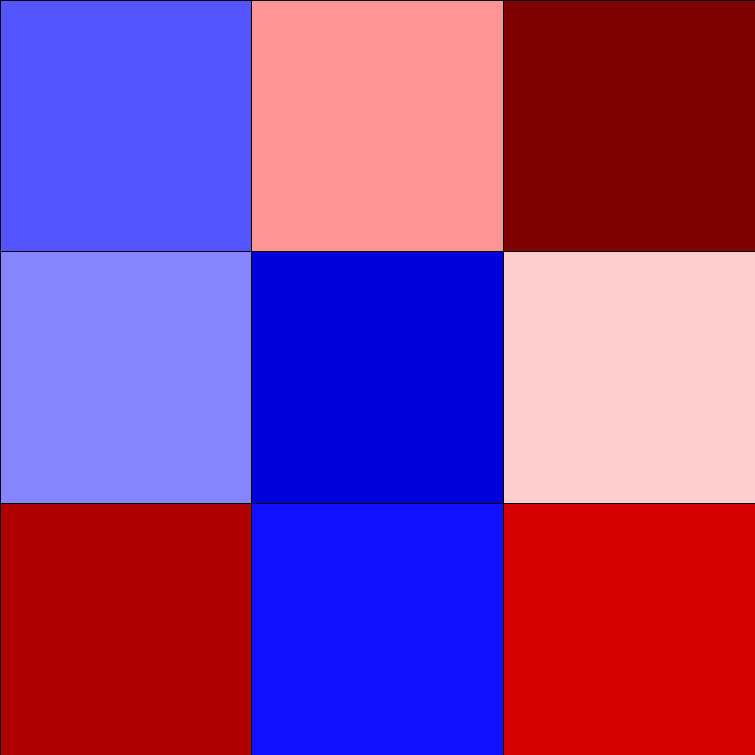}
        \\
        \includegraphics[width=0.32\textwidth]{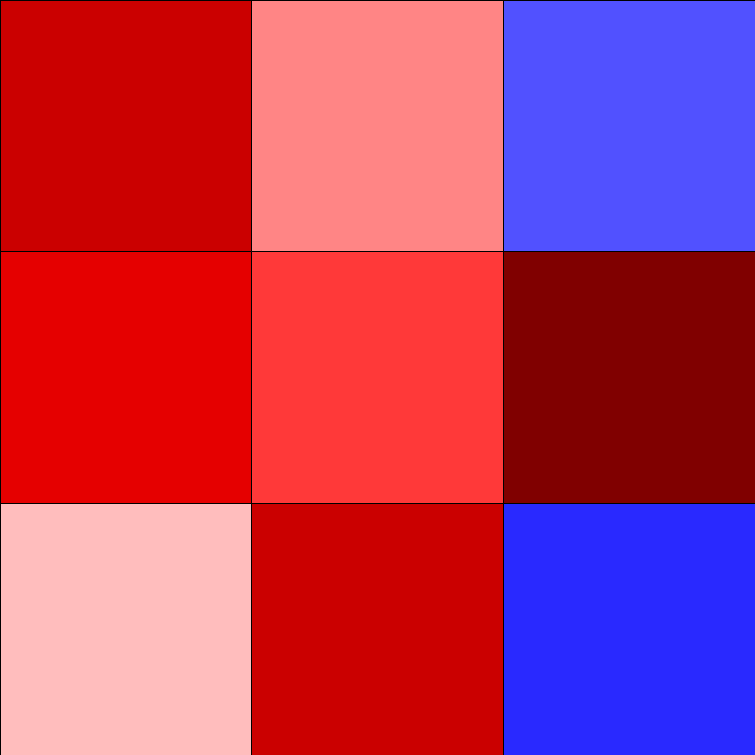}
        \includegraphics[width=0.32\textwidth]{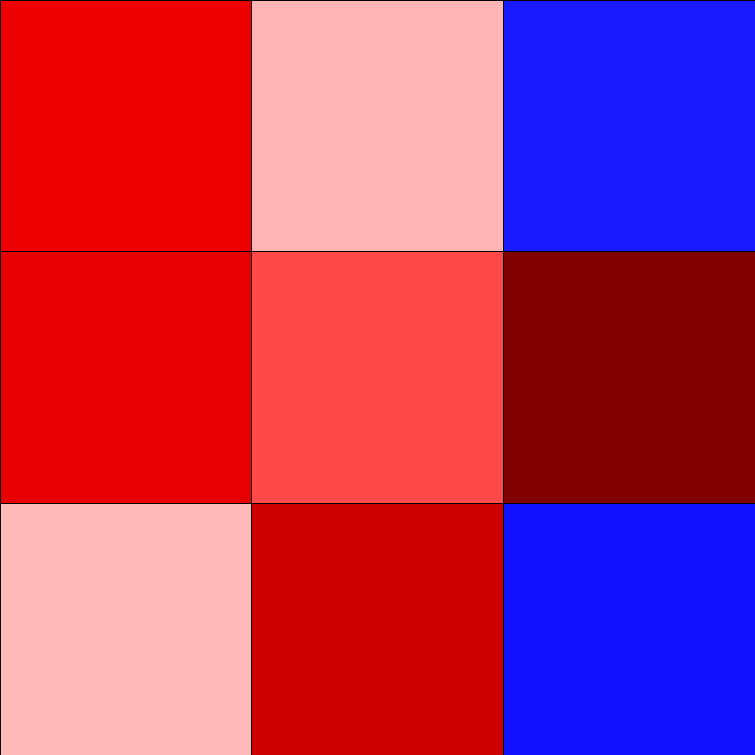}
        \includegraphics[width=0.32\textwidth]{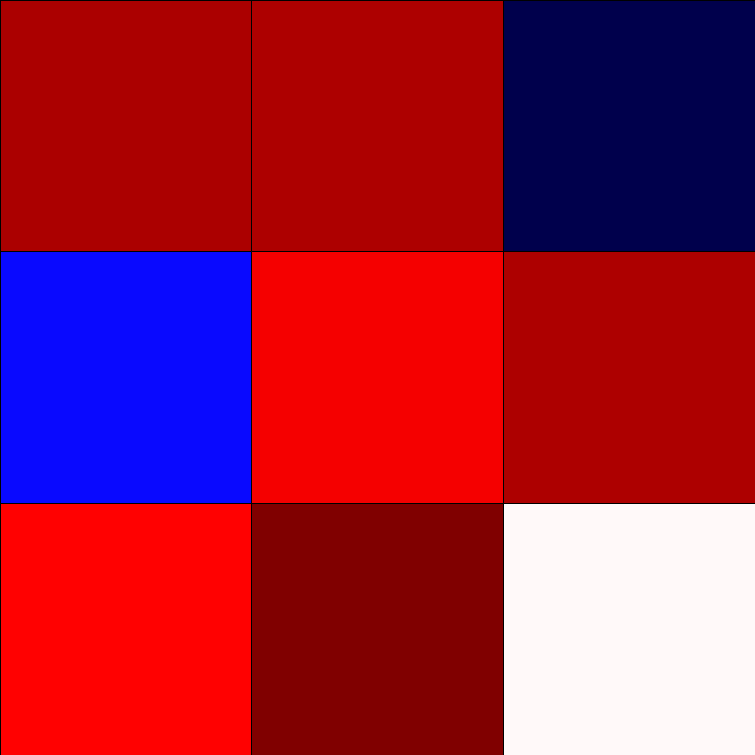}
    \end{subfigure}
    \quad
    \begin{subfigure}[t]{.3\textwidth}
        \centering
        \caption{\textbf{Last layer}}
        \footnotesize Initial \quad \ \ Small $\eta$ \quad \ \ Large $\eta$
        \includegraphics[width=0.32\textwidth]{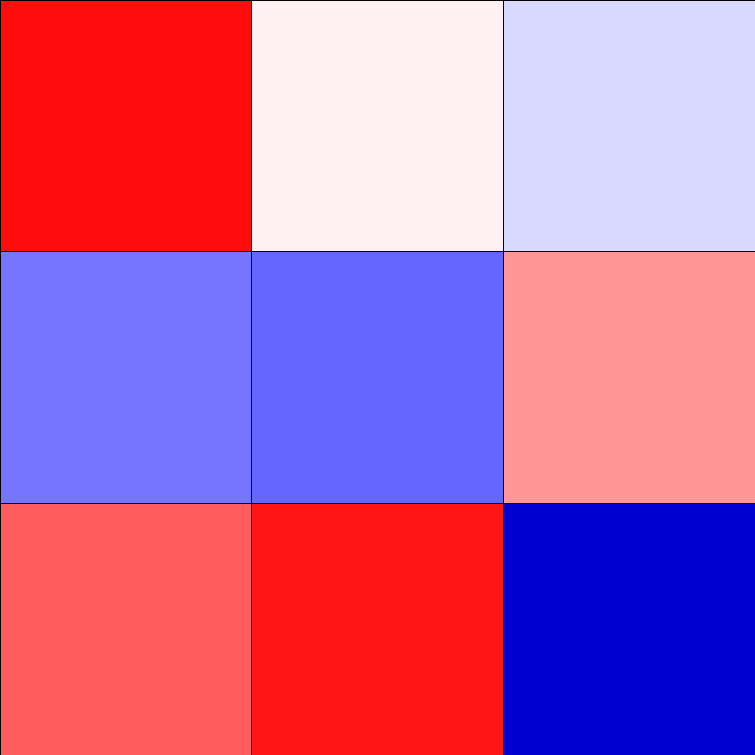}
        \includegraphics[width=0.32\textwidth]{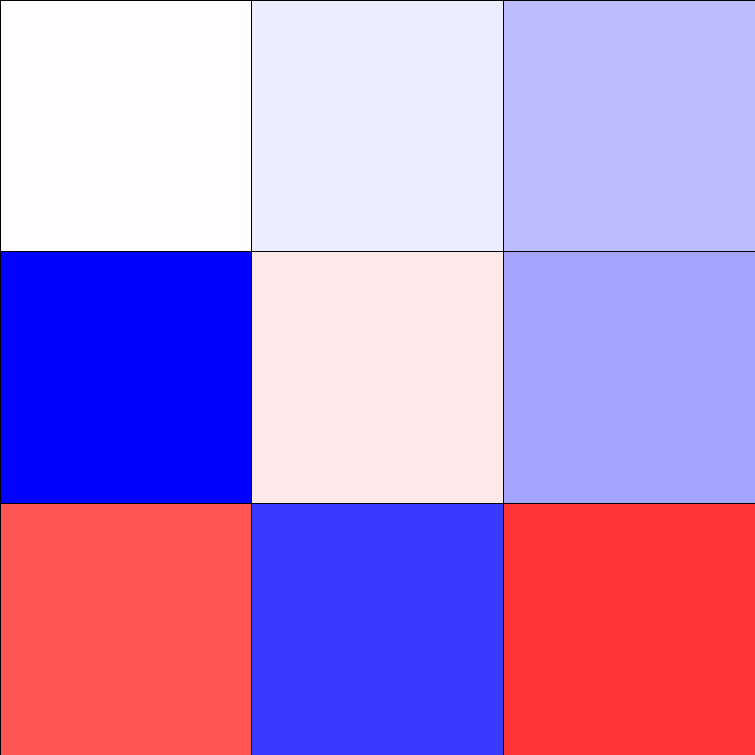}
        \includegraphics[width=0.32\textwidth]{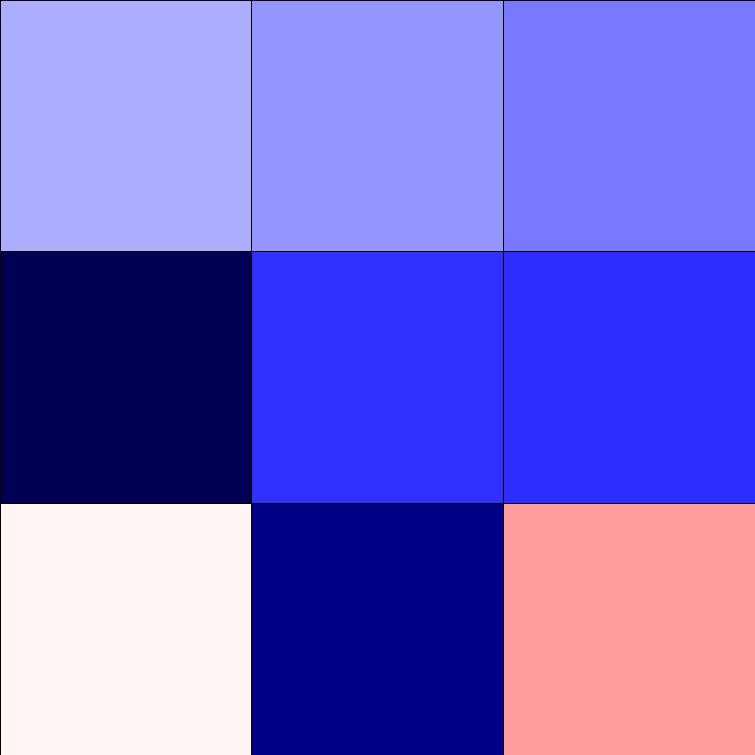}
        \\
        \includegraphics[width=0.32\textwidth]{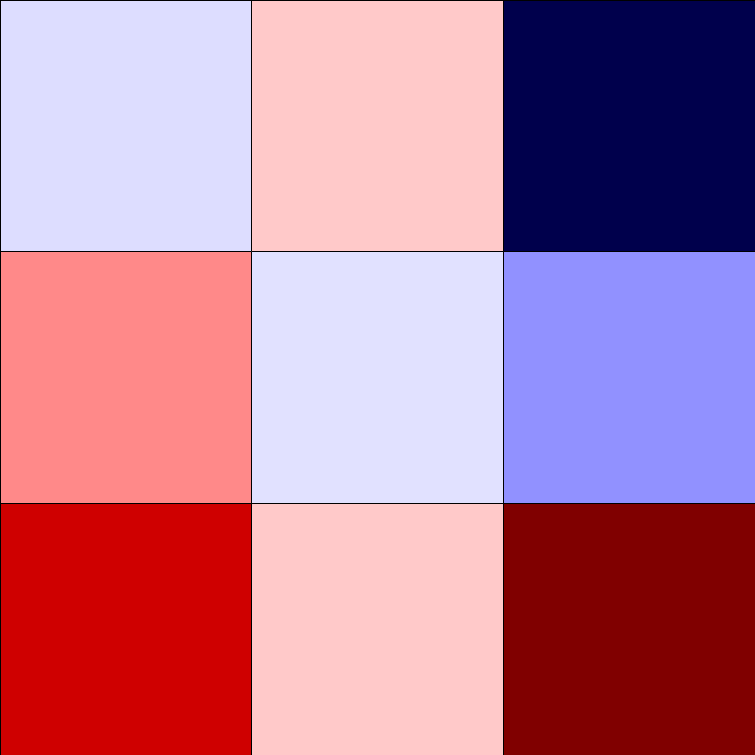}
        \includegraphics[width=0.32\textwidth]{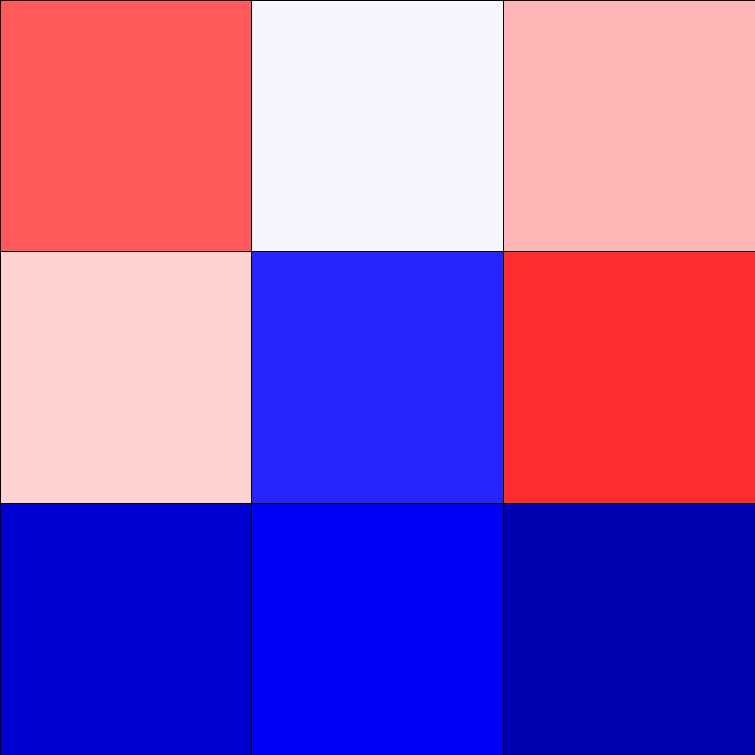}
        \includegraphics[width=0.32\textwidth]{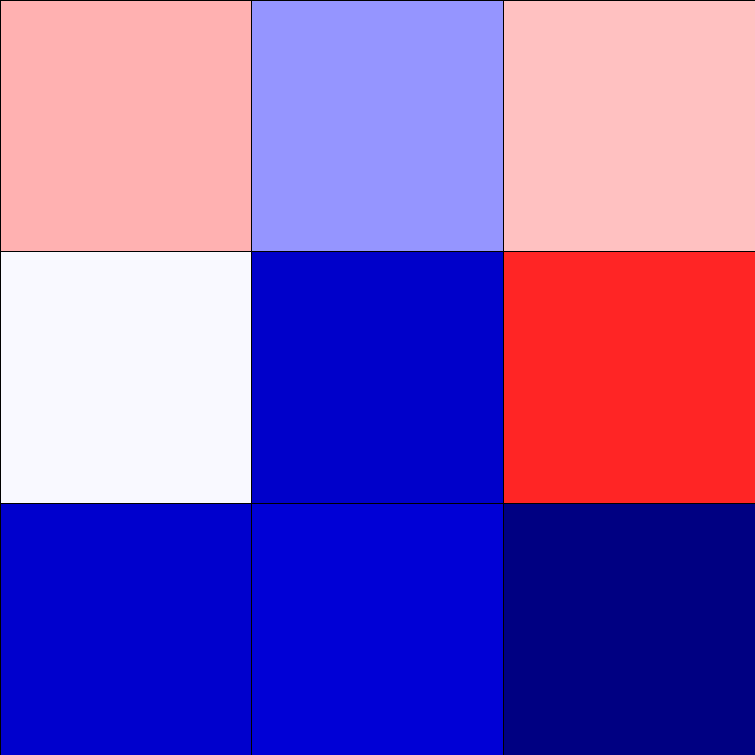}
        \\
        \includegraphics[width=0.32\textwidth]{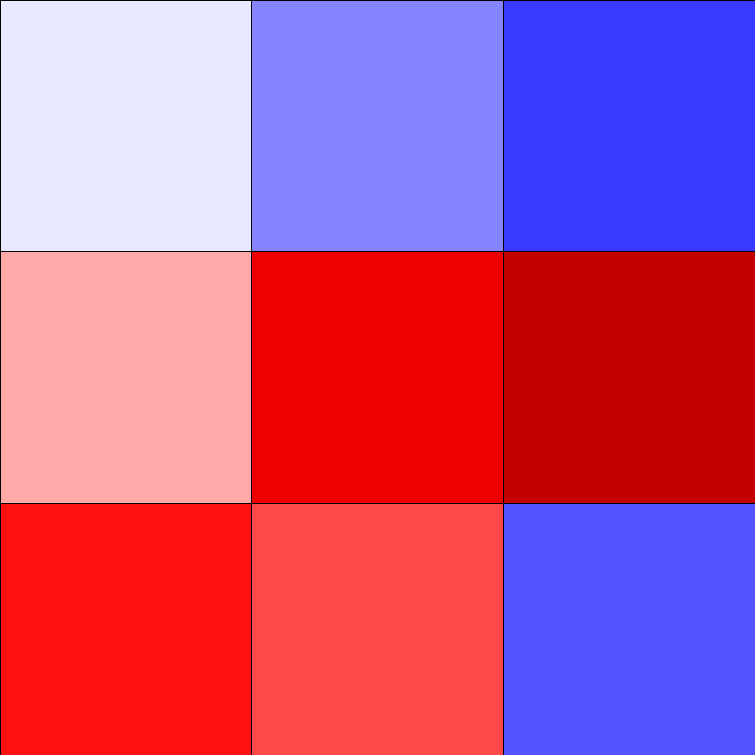}
        \includegraphics[width=0.32\textwidth]{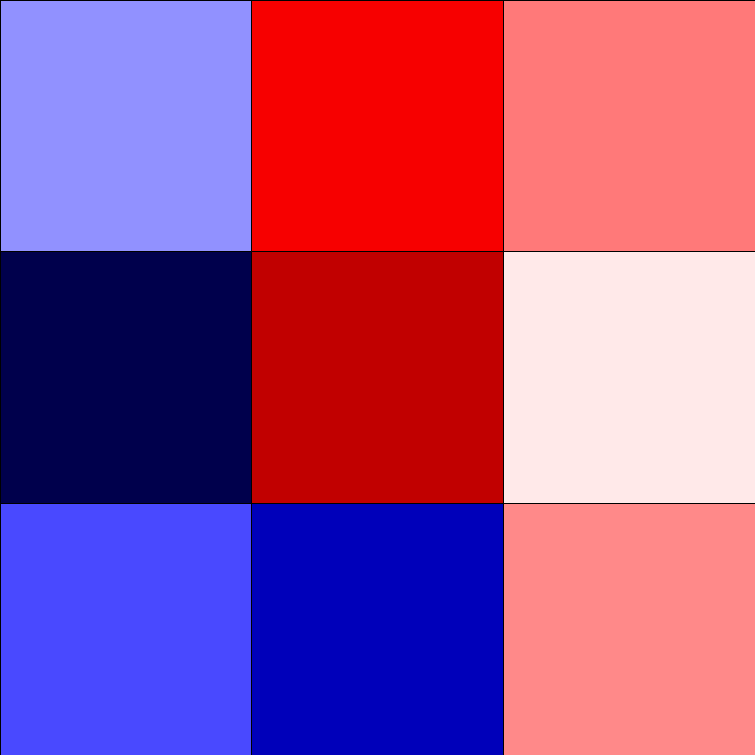}
        \includegraphics[width=0.32\textwidth]{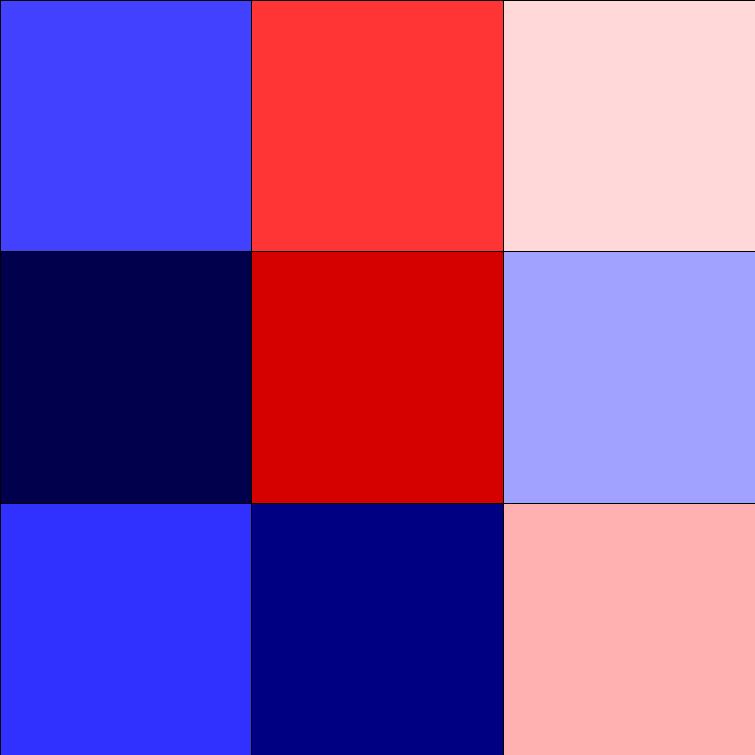}
        \\
        \includegraphics[width=0.32\textwidth]{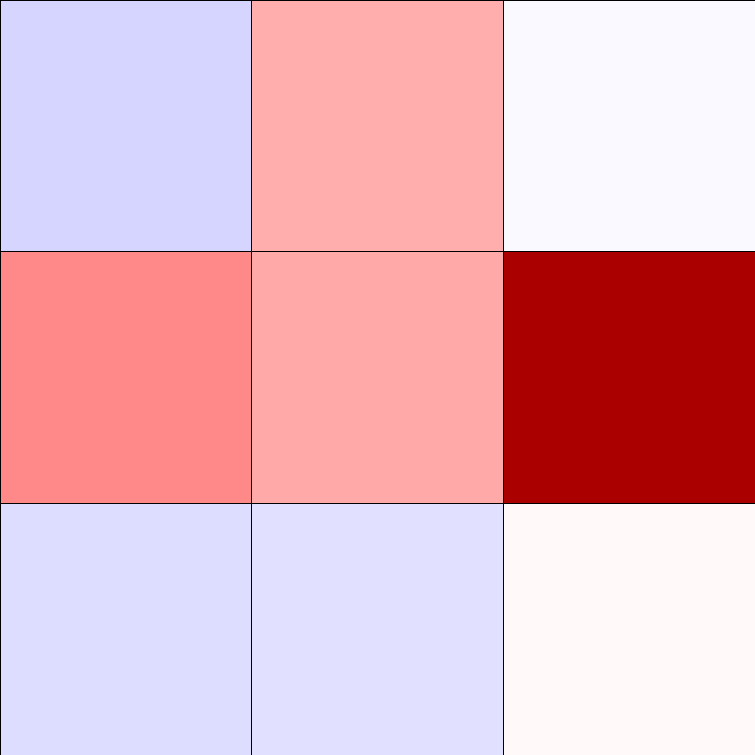}
        \includegraphics[width=0.32\textwidth]{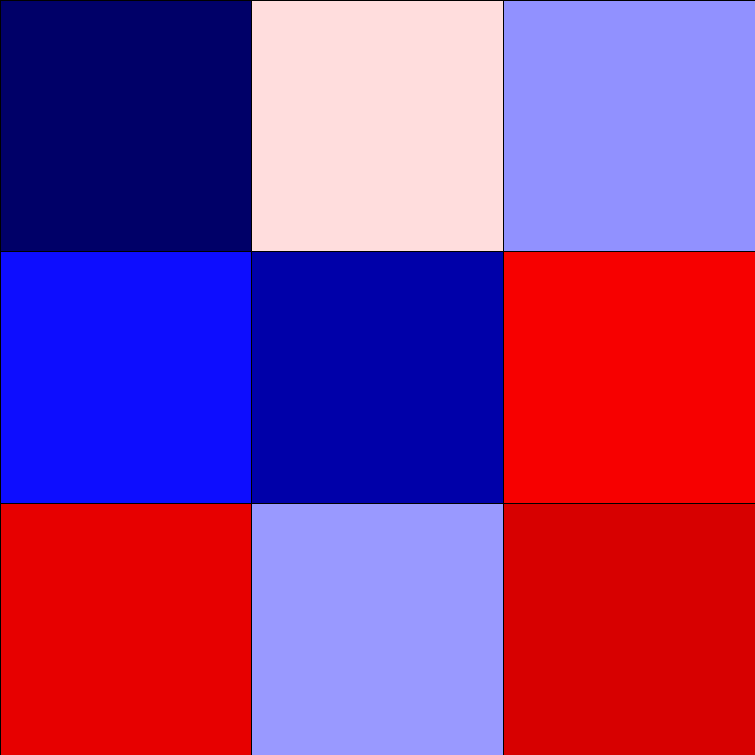}
        \includegraphics[width=0.32\textwidth]{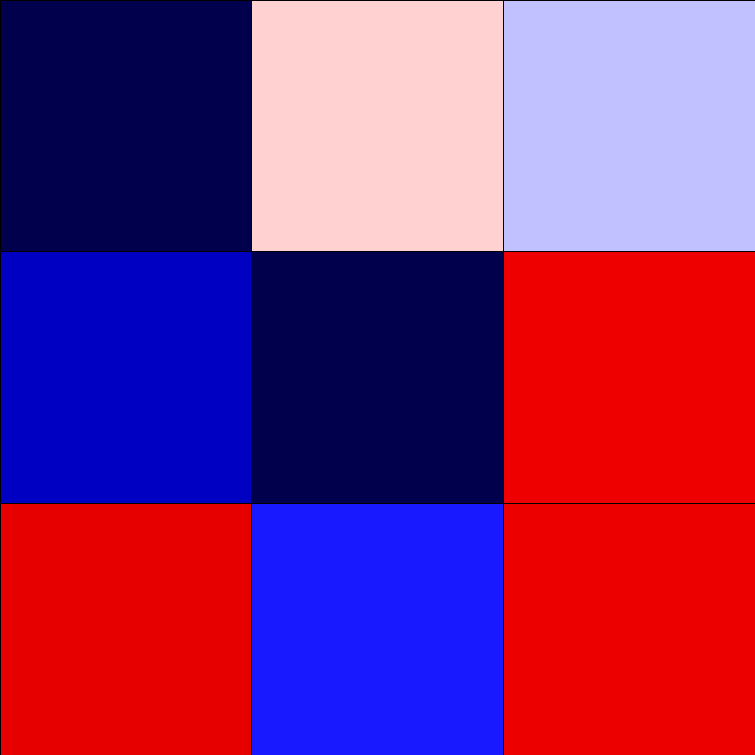}
    \end{subfigure}
    \caption{Visualization on four sets of convolutional filters taken from different layers of ResNets-18 trained on CIFAR-10 with small vs. large step size $\eta$ (the $50\%$ decay schedule). For small step sizes, the early and middle layers stay very close to randomly initialized ones which indicates the absence of feature learning.}
    \label{fig:feature_learning_conv_filters_cifar10}
\end{figure}

\end{document}